\documentclass[sigconf]{aamas} 
\usepackage[utf8]{inputenc}
\usepackage{mathtools}
\usepackage{ifthen}
\usepackage{braket}
\usepackage{amsthm}
\usepackage{url}
\usepackage{appendix}
\usepackage{hyperref}
\usepackage{subcaption}
\usepackage[shortlabels]{enumitem}
\usepackage{multibib}
\usepackage{xcolor}
\usepackage{enumitem}
\usepackage{wrapfig}
\usepackage{footmisc}
\usepackage{bm}

\usepackage{etoolbox} 

\makeatletter
\newcommand{\pushright}[1]{\ifmeasuring@#1\else\omit\hfill$\displaystyle#1$\fi\ignorespaces}
\newcommand{\pushleft}[1]{\ifmeasuring@#1\else\omit$\displaystyle#1$\hfill\fi\ignorespaces}
\makeatother

\newcommand{\fun}[1]{\ensuremath{\mathopen{}\mathclose\bgroup\left(#1\aftergroup\egroup\right)}}

\newcommand{\vect}[1]{\ensuremath{\bm{#1}}}

\newcommand{\mdp}{\ensuremath{\mathcal{M}}}
\newcommand{\states}{\ensuremath{\mathcal{S}}}
\newcommand{\actions}{\ensuremath{\mathcal{A}}}

\newcommand{\probtransitions}{\ensuremath{\mathbf{P}}} 
\newcommand{\rewards}{\ensuremath{\mathcal{R}}}

\newcommand{\sinit}{\ensuremath{\mu}}

\newcommand{\s}{\ensuremath{s}}

\newcommand{\action}{\ensuremath{a}}

\newcommand{\act}[1]{\ensuremath{\mathit{Act}\ifthenelse{\equal{#1}{}}{}{(#1)}}}

\newcommand{\policy}{\ensuremath{\pi}}
\newcommand{\policies}{\ensuremath{\Pi}}

\newcommand{\stationary}[1]{\ensuremath{\xi_{#1}}}

\newcommand{\vecrewards}{\ensuremath{\vect{\rewards}}}
\newcommand{\accrewards}[1]{\ensuremath{\vv^\text{acc}_{#1}}}

\newcommand{\momdp}{\ensuremath{\vect{\mathcal{M}}}}

\newcommand{\momdptuple}{\langle \states, \actions, \probtransitions, \vecrewards, \allowbreak \sinit, \discount \rangle}

\newcommand{\nadir}{\ensuremath{\vv^\text{n}}}
\newcommand{\ideal}{\ensuremath{\vv^\text{i}}}
\newcommand{\vpi}{\ensuremath{\vv^\policy}}

\newcommand{\vast}{\ensuremath{\vv^\ast}}

\newcommand{\pd}{\ensuremath{\succ}}
\newcommand{\pde}{\ensuremath{\succeq}}

\newcommand{\npde}{\ensuremath{\nsucceq}}

\newcommand{\pf}[2][]{\ensuremath{\mathcal{V}_{#1}^{#2}}}
\newcommand{\truepf}{\ensuremath{\pf{\ast}}}
\newcommand{\approxpf}{\ensuremath{\pf{\tau}}}

\newcommand{\interior}{\ensuremath{\text{int }}}
\newcommand{\volume}{\ensuremath{\text{vol }}}
\newcommand{\boundary}{\ensuremath{\partial}}

\newcommand{\dist}{\ensuremath{\text{dist}}}

\newcommand{\sr}{\ensuremath{s_{\vect{r}}}}

\newcommand{\rboundary}{\ensuremath{\boundary^r}}
\newcommand{\rpos}{\ensuremath{\mathbb{R}^d_{\geq 0}}}
\newcommand{\rapprox}{\ensuremath{\mathbb{R}^d_{\delta}}}
\newcommand{\bbox}{\ensuremath{\mathcal{B}}}
\newcommand{\vset}{\ensuremath{\sV}}
\newcommand{\cset}{\ensuremath{\sC}}
\newcommand{\dset}{\ensuremath{\sD}}
\newcommand{\iset}{\ensuremath{\sI}}
\newcommand{\lset}{\ensuremath{\sL}}
\newcommand{\uset}{\ensuremath{\sU}}
\newcommand{\oracle}{\ensuremath{\Omega^\tau}}

\newcommand{\observationfn}{\ensuremath{\mathcal{O}}}

 \newcommand{\encoderparameter}{\ensuremath{}}

\newcommand{\discount}{\ensuremath{\gamma}}

\newcommand{\Prob}{\ensuremath{\mathbb{P}}}

\newcommand{\expectedsymbol}[1]{\ensuremath{\mathop{\mathbb{E}}\ifthenelse{\equal{#1}{}}{}{_{#1}}}}
\newcommand{\expected}[2]{\ensuremath{\expectedsymbol{#1} \left[ #2 \right]}}

\newcommand{\normal}[3]{\ensuremath{\displaystyle \ifthenelse{\equal{#3}{}}{\mathcal{N}(#1, #2)}{\mathcal{N}(#3\,;\, #1, #2)}}}

\newcommand{\overbar}[1]{\mkern 1.5mu\overline{\mkern-1.5mu#1\mkern-1.5mu}\mkern 1.5mu}
\newcommand{\overbarit}[1]{\,\overline{\!{#1}}}
\newcommand{\embed}{\ensuremath{\phi}}

\newcommand{\latentprobtransitions}{\ensuremath{\overbar{\probtransitions}}}

\newcommand{\latentrewards}{\ensuremath{\overbarit{\rewards}}}

\newcommand{\latentbeliefupdate}{\ensuremath{\overbar{\tau}}}

\newcommand{\localtransitionloss}[1]{L_{\probtransitions}}
\newcommand{\localrewardloss}[1]{L_{\rewards}}
\newcommand{\observationloss}[1]{\ensuremath{L_{\observationfn}}}
\newcommand{\beliefloss}[1]{\ensuremath{L_{\latentbeliefupdate}}}
\newcommand{\onpolicyrewardloss}[1]{\ensuremath{L_{\latentrewards}^{\varphi}}}
\newcommand{\onpolicytransitionloss}[1]{\ensuremath{L_{\latentprobtransitions}^{\varphi}}}

\newcommand{\KR}[1]{\ensuremath{\ifthenelse{\equal{#1}{}}{K_{\latentrewards}}{K_{\latentrewards}^{#1}}}}
\newcommand{\KP}[1]{\ensuremath{\ifthenelse{\equal{#1}{}}{K_{\latentprobtransitions}}{K_{\latentprobtransitions}^{#1}}}}

\newcommand{\originaltolatentstationary}[1]{{\latentprobtransitions_{\embed_{\encoderparameter}\stationary{\ifthenelse{\equal{#1}{}}{\policy}{#1}}}}}

\def\ceil#1{\lceil #1 \rceil}

\def\1{\bm{1}}

\def\vc{{\bm{c}}}

\def\vi{{\bm{i}}}

\def\vl{{\bm{l}}}

\def\vr{{\bm{r}}}

\def\vu{{\bm{u}}}
\def\vv{{\bm{v}}}

\def\vx{{\bm{x}}}
\def\vy{{\bm{y}}}

\DeclareMathAlphabet{\mathsfit}{\encodingdefault}{\sfdefault}{m}{sl}
\SetMathAlphabet{\mathsfit}{bold}{\encodingdefault}{\sfdefault}{bx}{n}

\def\sC{{\mathcal{C}}}
\def\sD{{\mathcal{D}}}

\def\sI{{\mathcal{I}}}

\def\sL{{\mathcal{L}}}

\def\sS{{\mathcal{S}}}

\def\sU{{\mathcal{U}}}
\def\sV{{\mathcal{V}}}

\def\sX{{\mathcal{X}}}
\def\sY{{\mathcal{Y}}}

\newcommand{\normmax}{L^\infty}

\DeclareMathOperator*{\argmax}{arg\,max}

\usepackage{balance} 

\usepackage[utf8]{inputenc} 
\usepackage[T1]{fontenc}    
\usepackage{hyperref}       
\usepackage{url}            
\usepackage{booktabs}       
\usepackage{nicefrac}       
\usepackage{microtype}      
\usepackage{xcolor}         

\usepackage{subcaption}
\usepackage{algorithm}
\usepackage{algorithmicx}
\usepackage[noend]{algpseudocode}

\usepackage[capitalize,noabbrev]{cleveref}

\theoremstyle{plain}
\newtheorem{theorem}{Theorem}[section]
\newtheorem*{theorem53}{Theorem 5.3}
\newtheorem*{theorem54}{Theorem 5.4}
\newtheorem*{theorem41}{Theorem 4.1}

\newtheorem*{theorem43}{Theorem 4.3}
\newtheorem{proposition}[theorem]{Proposition}
\newtheorem{lemma}[theorem]{Lemma}
\newtheorem{corollary}[theorem]{Corollary}

\theoremstyle{definition}
\newtheorem{definition}[theorem]{Definition}
\newtheorem{assumption}[theorem]{Assumption}
\theoremstyle{remark}

\newenvironment{proofsketch}{%
\renewcommand{\proofname}{Proof sketch}\proof}{\endproof}

\newcommand{\smallparagraph}[1]{\smallskip\noindent\textbf{#1}.}

\makeatletter
\gdef\@copyrightpermission{
  \begin{minipage}{0.2\columnwidth}
   \href{https://creativecommons.org/licenses/by/4.0/}{\includegraphics[width=0.90\textwidth]{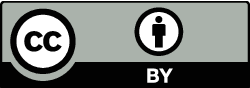}}
  \end{minipage}\hfill
  \begin{minipage}{0.8\columnwidth}
   \href{https://creativecommons.org/licenses/by/4.0/}{This work is licensed under a Creative Commons Attribution International 4.0 License.}
  \end{minipage}
  \vspace{5pt}
}
\makeatother

\setcopyright{ifaamas}
\acmConference[AAMAS '25]{Proc.\@ of the 24th International Conference
on Autonomous Agents and Multiagent Systems (AAMAS 2025)}{May 19 -- 23, 2025}
{Detroit, Michigan, USA}{Y.~Vorobeychik, S.~Das, A.~Nowé  (eds.)}
\copyrightyear{2025}
\acmYear{2025}
\acmDOI{}
\acmPrice{}
\acmISBN{}

\acmSubmissionID{302}

\title[Provably Unveiling the Pareto Front with MORL]{Divide and Conquer: Provably Unveiling the Pareto Front with Multi-Objective Reinforcement Learning}

\author{Willem Röpke}
\affiliation{
  \institution{Vrije Universiteit Brussel}
  \city{Brussels}
  \country{Belgium}
}
\email{willem.ropke@vub.be}

\author{Mathieu Reymond}
\affiliation{
  \institution{Université de Montréal,\\ Mila - Quebec AI Institute}
  \city{Montreal}
  \country{Canada}
}

\affiliation{
  \institution{Vrije Universiteit Brussel}
  \city{Brussels}
  \country{Belgium}
}

\author{Patrick Mannion}
\affiliation{
  \institution{University of Galway}
  \city{Galway}
  \country{Ireland}
}

\author{Diederik M. Roijers}
\affiliation{
  \institution{City of Amsterdam}
  \city{Amsterdam}
  \country{the Netherlands}
}
\affiliation{
  \institution{Vrije Universiteit Brussel}
  \city{Brussels}
  \country{Belgium}
}

\author{Ann Nowé}
\affiliation{
  \institution{Vrije Universiteit Brussel}
  \city{Brussels}
  \country{Belgium}
}

\author{Roxana R\u{a}dulescu}
\affiliation{
  \institution{Utrecht University}
  \city{Utrecht}
  \country{the Netherlands}
}
\affiliation{
  \institution{Vrije Universiteit Brussel}
  \city{Brussels}
  \country{Belgium}
}

\begin{abstract}
An important challenge in multi-objective reinforcement learning is obtaining a Pareto front of policies to attain optimal performance under different preferences. We introduce Iterated Pareto Referent Optimisation (IPRO), which decomposes finding the Pareto front into a sequence of constrained single-objective problems. This enables us to guarantee convergence while providing an upper bound on the distance to undiscovered Pareto optimal solutions at each step. We evaluate IPRO using utility-based metrics and its hypervolume and find that it matches or outperforms methods that require additional assumptions. By leveraging problem-specific single-objective solvers, our approach also holds promise for applications beyond multi-objective reinforcement learning, such as planning and pathfinding.
\end{abstract}

\keywords{Reinforcement learning; Multi-objective; Pareto front}

\newcommand{\BibTeX}{\rm B\kern-.05em{\sc i\kern-.025em b}\kern-.08em\TeX}

\begin{document}


\pagestyle{fancy}
\fancyhead{}


\maketitle 


\section{Introduction}
In sequential decision-making problems, agents often have multiple and conflicting objectives. Controlling a water reservoir, for example, involves a complex trade-off between environmental, economic and social factors \citep{castelletti2013multiobjective}. Because the objectives are conflicting, decision-makers ultimately need to make a suitable trade-off. In such situations, multi-objective reinforcement learning (MORL) can be used to compute a set of candidate optimal policies that offer the best available trade-offs, empowering decision-makers to select their preferred policy~\citep{hayes2022practical}.

We focus on learning the Pareto front, which comprises all policies yielding non-dominated expected returns. When assuming decision-makers employ a linear scalarisation function or allow stochastic policies, the Pareto front is guaranteed to be convex \citep{roijers2017multiobjective}, facilitating the use of effective solution methods \citep{yang2019generalized,xu2020predictionguided}. However, deterministic policies are often preferred for reasons of safety, accountability, or interpretability, and in such cases, the Pareto front may exhibit concave regions. Algorithms addressing this setting have been elusive, with successful solutions limited to purely deterministic environments \citep{reymond2022pareto}.

To tackle general policy classes and MOMDPs, we propose Iterated Pareto Referent Optimisation (IPRO), which decomposes this task into a sequence of constrained single-objective problems. In multi-objective optimisation (MOO), decomposition stands as a successful paradigm for computing a Pareto front. This approach makes use of efficient single-objective methods to solve the decomposed problems, thereby also establishing a robust connection between advances in multi-objective and single-objective methods \citep{zhang2007moea}. In particular, existing MORL algorithms dealing with a convex Pareto front frequently employ decomposition and rely on single-objective RL algorithms to solve the resulting problems~\citep{lu2023multiobjective,alegre2023sampleefficient}.

\vspace{-0.1em}
\smallparagraph{Contributions} IPRO is an anytime algorithm that decomposes learning the Pareto front into learning a sequence of Pareto optimal policies. We show that learning a Pareto optimal policy corresponds to a constrained single-objective problem for which principled solution methods are derived. Combining these, we guarantee convergence to the Pareto front and provide bounds on the distance to undiscovered solutions at each iteration. Our complexity analysis shows that IPRO requires a polynomial number of iterations to approximate the Pareto front for a constant number of objectives. While IPRO applies to any policy class, we specifically demonstrate its effectiveness for deterministic policies, a class lacking general methods. When comparing IPRO to algorithms that require additional assumptions on the structure of the Pareto front or the underlying environment, we find that it matches or outperforms them, thereby showcasing its efficacy. 

\section{Related work} 
When learning a single policy in MOMDPs, as is necessary in IPRO, conventional methods often adapt single-objective RL algorithms. For example, \citet{siddique2020learning} extend DQN, A2C and PPO to learn a fair policy by optimising the generalised Gini index of the expected returns. \citet{reymond2023actorcritic} extend this to general non-linear functions and establish a policy gradient theorem for this setting. When maximising a concave function of the expected returns, efficient methods exist that guarantee global convergence~\citep{zhang2020variational,zahavy2021reward,geist2022concave}.

Decomposition is a promising technique for MORL due to its ability to leverage strong single-objective methods as a subroutine \citep{felten2024multiobjective}. When the Pareto front is convex, many techniques rely on the fact that it can be decomposed into a sequence of single-objective RL problems where the scalar reward is a convex combination of the original reward vector \citep{yang2019generalized,alegre2023sampleefficient}. When the Pareto front is non-convex, \citet{vanmoffaert2013scalarized} learn deterministic policies on the Pareto front by decomposing the problem using the Chebyshev scalarisation function but do not provide any theoretical guarantees and only evaluate on discrete settings.

In MOO, a related methodology was proposed by \citet{legriel2010approximating} to obtain approximate Pareto fronts. Their approach iteratively proposes queries to an oracle and uses the return value to trim sections from the search space. In contrast, we introduce an alternative technique for query selection that ensures convergence to the \emph{exact} Pareto front and aims to minimise the number of iterations. Moreover, we introduce a procedure that deals with imperfect oracles and contribute novel results that are particularly useful for MORL. 


\section{Preliminaries}
\label{sec:preliminaries}
\smallparagraph{Pareto dominance} For two vectors $\vv, \vv' \in \mathbb{R}^d$ we say that $\vv$ Pareto dominates $\vv'$, denoted $\vv \pd \vv'$, when $\forall j \in \{1, \dotsc, d\}: v_j \geq v'_j$ and $\vv \neq \vv'$. When dropping the second condition, we write $\vv \pde \vv'$. We say that $\vv$ strictly Pareto dominates $\vv'$, denoted $\vv > \vv'$ when $\forall j \in \{1, \dotsc, d\}: v_j > v'_j$. When a vector is not pairwise Pareto dominated, it is Pareto optimal. A vector is weakly Pareto optimal whenever there is no other vector that strictly Pareto dominates it. 

In multi-objective decision-making, Pareto optimal vectors are relevant when considering decision-makers with monotonically increasing utility functions. In particular, if $\vv \pd \vv'$, then $\vv$ will be preferred over $\vv'$ by all decision-makers. The set of all pairwise Pareto non-dominated vectors is called the Pareto front, denoted $\truepf$, and an approximate Pareto front $\approxpf$ with tolerance $\tau$ is an approximation to $\truepf$ such that $\forall \vv \in \truepf, \exists \vv' \in \approxpf: \|\vv - \vv'\|_\infty \leq \tau$. We refer to the least upper bound of the Pareto front as the ideal $\ideal$, and the greatest lower bound as the nadir $\nadir$ (see \cref{fig:ipro-sets}). 

\smallparagraph{Achievement scalarising functions} Achievement scalarising functions (ASFs) scalarise a multi-objective problem such that an optimal solution to the single-objective problem is (weakly) Pareto optimal \citep{miettinen1998nonlinear}. These functions are parameterised by a reference point $\vr$, also called the referent. Points dominating the referent form the target region. ASFs are classified into two types: order representing and order approximating. An ASF $\sr$ is order representing when it is strictly increasing, i.e. $\vv > \vv' \implies \sr(\vv) > \sr(\vv')$, and only returns non-negative values for $\vv$ when $\vv \pde \vr$. An ASF is order approximating when it is strongly increasing, i.e. $\vv \pd \vv' \implies \sr(\vv) > \sr(\vv')$, but may assign non-negative values to solutions outside the target region. An ASF cannot be both strongly increasing and exclusively non-negative within the target region \citep{wierzbicki1982mathematical}.
 
As an example, consider two vectors $\vv_1 = (1, 2)$ and $\vv_2 = (1, 1)$ where $\vv_1$ Pareto dominates $\vv_2$ ($\vv_1 \pd \vv_2$) but does not strictly dominate it. With a strictly increasing ASF $\sr$, it is possible that $\sr(\vv_1) = \sr(\vv_2)$. However, a strongly increasing ASF ensures that $\sr(\vv_1) > \sr(\vv_2)$. Consequently, maximising an order representing ASF guarantees a weakly Pareto optimal solution inside the target region, while maximising an order approximating ASF guarantees a Pareto optimal solution, though this solution might lie outside the target region. We employ the augmented Chebyshev scalarisation function, a frequently used ASF \citep{nikulin2012new,vanmoffaert2013scalarized}.

\smallparagraph{Problem setup}
We consider sequential multi-objective decision-making problems, modelled as a multi-objective Markov decision process (MOMDP). A MOMDP is a tuple $\momdp = \momdptuple$ where $\states$ is a set of states, $\actions$ a set of actions, $\probtransitions$ a transition function, $\vecrewards: \states \times \actions \times \states \to \mathbb{R}^d$ a vectorial reward function with $d \geq 2$ the number of objectives, $\mu$ a distribution over initial states and $\gamma$ a discount factor. Since there is generally not a single policy that maximises the expected return for all objectives, we introduce a partial ordering over policies on the basis of Pareto dominance. We say that a policy $\policy \in \policies$ Pareto dominates another if its expected return, defined as $\vpi := \expected{\policy, \sinit}{\sum_{t=0}^\infty \gamma^t \vecrewards(\s_t, \action_t, \s_{t+1})}$, Pareto dominates the expected return of the other policy.

Our goal is to learn a Pareto front of memory-based deterministic policies in MOMDPs. Such policies are relevant in safety-critical settings, where stochastic policies may have catastrophic outcomes but can Pareto dominate deterministic policies \citep{delgrange2020simple}. Furthermore, for deterministic policies, it can be shown that memory-based policies may Pareto dominate stationary policies \citep{roijers2017multiobjective}. In this setting, it is known that the Pareto front may be non-convex and thus cannot be fully recovered by methods based on linear scalarisation. Furthermore, to the best of our knowledge, no algorithm exists that produces a Pareto front for such policies in general MOMDPs.

\section{Iterated Pareto referent optimisation}
\label{sec:ipro}

\begin{figure*}[tb]
    \centering
        \begin{subfigure}[b]{0.329\textwidth}
        \centering
        \includegraphics[trim={0.5cm 0.3cm 0.1cm 0.22cm},clip,width=\textwidth]{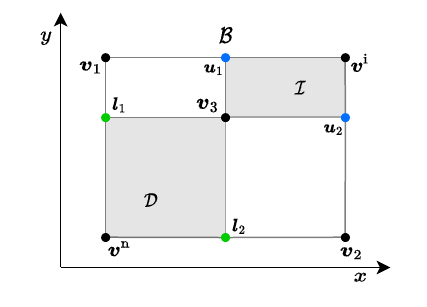}
        \subcaption{\vspace{-1em}}
        \label{fig:ipro-stage1}
    \end{subfigure}
    \begin{subfigure}[b]{0.329\textwidth}
        \centering
        \includegraphics[trim={0.5cm 0.3cm 0.1cm 0.22cm},clip,width=\textwidth]{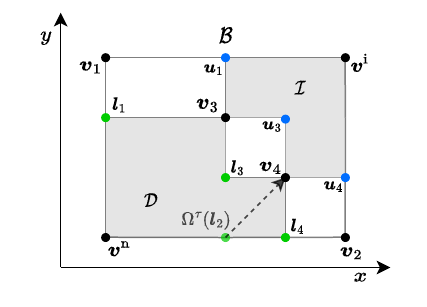}
        \subcaption{\vspace{-1em}}
        \label{fig:ipro-stage2}
    \end{subfigure}
    \begin{subfigure}[b]{0.329\textwidth}
        \centering
        \includegraphics[trim={0.5cm 0.3cm 0.1cm 0.22cm},clip,width=\textwidth]{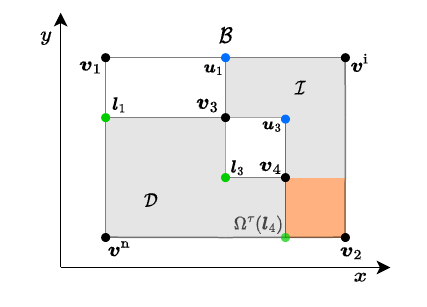}
        \subcaption{\vspace{-1em}}
        \label{fig:ipro-stage3}
    \end{subfigure}
    \Description{A figure with three subplots showcasing an example execution of IPRO.}
    \caption{(\textbf{a}) The bounding box $\bbox$, defined by the nadir $\nadir$ and ideal $\ideal$, contains all Pareto optimal solutions. The dominated set $\dset$ and infeasible set $\iset$ are defined by the current approximation to the Pareto front $\vset = \{\vv_1, \vv_2, \vv_3\}$ and are shaded. The lower bounds $\vl \in \lset$ are highlighted in green, while the upper bounds $\vu \in \uset$ are highlighted in blue. (\textbf{b})~After querying the Pareto oracle $\oracle$ with $\vl_2$, $\vv_4$ is added to the Pareto front and $\lset$ and $\uset$ are updated to represent the new corners of $\dset$ and $\iset$ respectively. (\textbf{c}) When the Pareto oracle cannot find a feasible solution strictly dominating $\vl_4$, it is added to the completed set $\cset$ and the shaded orange area is added to the infeasible set $\iset$.}
    \label{fig:ipro-sets}
\end{figure*}

We present Iterated Pareto Referent Optimisation (IPRO) to learn a Pareto front in MOMDPs. IPRO generates a sequence of constrained single-objective problems while retaining a set of guaranteed lower and upper bounds to the Pareto front. An example execution of IPRO is illustrated in \cref{fig:ipro-sets}. Formal proofs for all theoretical results are provided in \cref{ap:ipro-proofs}.

\subsection{Algorithm overview}
The core idea of IPRO is to bound the search space that may contain value vectors corresponding to Pareto optimal policies and iteratively remove sections from this space. This is achieved by leveraging an oracle to obtain a policy with its value vector in some target region and utilising this to update the boundaries of the search space. Detailed pseudocode is given in \cref{alg:ipro}.

\begin{algorithm}[t]
\caption{The IPRO algorithm.}
\label{alg:ipro}
    \begin{algorithmic}[1]
    \Require A Pareto oracle $\oracle$ with tolerance $\tau$ 
    \Ensure A $\tau$-Pareto front $\vset$
    \State Get maximal points $\{\vv^1, \dotsc, \vv^d\}$ to create the ideal $\ideal$
    \State Get minimal points to estimate the nadir $\nadir$
    \State Form a bounding box $\bbox$ from $\nadir$ and $\ideal$
    \State $\uset \gets \{\ideal\}$, $\lset \gets \{\nadir\}$
    \State $\vset \gets \{\vv^1, \dotsc, \vv^d\}$ and $\cset \gets \emptyset$
    \For{$\vv \in \{\vv^1, \dotsc, \vv^d\}$}
        \State $\lset \gets \textproc{update}(\vv, \lset)$
    \EndFor
    \While{$\max_{\vu \in \uset} \min_{\vv' \in \vset} \|\vu - \vv'\|_\infty >  \tau$}
    \State $\vl \gets \textsc{Select}(\lset)$
    \State $\textsc{success}, \vast \gets \oracle(\vl)$
    \If{$\textsc{success}$}
        \State $\vset \gets \vset \cup \{\vast\}$
        \State $\lset \gets \textproc{update}(\vast, \lset)$, $\uset \gets -\textproc{update}(-\vast, -\uset)$
    \Else{}
        \State $\cset \gets \cset \cup \{\vl\}$ 
        \State $\lset \gets \lset \setminus \{\vl\}$, $\uset \gets -\textproc{update}(-\vl, -\uset)$
    \EndIf
    \EndWhile
    \Procedure{update}{$\vast, \sX$}
        \State $\sX' \gets \{\}$
        \For{$\vv \in \sX$}
        \If{$\vast > \vv$}
        \State $\sX' \gets \sX' \cup \{(\vv_{-j}, \vast_j) \mid j \in [d] \}$
        \Else
        \State $\sX' \gets \sX' \cup \{\vv\}$
        \EndIf
        \EndFor
        \State $\sX' \gets \textsc{Prune}(\sX')$
    \EndProcedure
    \end{algorithmic}
\end{algorithm}

\smallparagraph{Bounding the search space} 
It is necessary to bound the space in which Pareto non-dominated solutions may exist. By definition, the box spanned by the nadir $\nadir$ and ideal $\ideal$ contains all such points (shown as $\bbox$ in \cref{fig:ipro-sets}). We obtain the ideal by maximising each objective independently, effectively reducing the MOMDP to a regular MDP. The solutions constituting the ideal are further used to instantiate the Pareto front $\vset$. Since obtaining the nadir is generally more complicated \citep{miettinen1998nonlinear}, we compute a lower bound of the nadir by minimising each objective independently, analogous to the instantiation of the ideal. 

\smallparagraph{Obtaining a Pareto optimal policy} To obtain individual Pareto optimal policies we introduce a \emph{Pareto oracle} (fully formalised in \cref{sec:pareto-oracle}). Informally, a Pareto oracle $\oracle$ with tolerance $\tau$ takes a referent $\vr$ as input and attempts to return a weakly Pareto optimal policy $\policy$ whose expected return $\vpi$ strictly dominates the referent, i.e. $\vpi > \vr$. The oracle's output guides IPRO in deciding which points may still correspond to Pareto optimal policies. If the oracle succeeds (\cref{fig:ipro-stage2}), $\vpi$ is guaranteed to be weakly Pareto optimal, meaning all points dominated by $\vpi$ can be discarded, while all points strictly dominating $\vpi$ are infeasible, as otherwise $\policy$ would not have been returned. If the evaluation fails (\cref{fig:ipro-stage3}), all points strictly dominating $\vr$ can be excluded as they are either infeasible or within tolerance $\tau$. This mechanism ensures efficient exploration of the Pareto front by eliminating infeasible or dominated regions.

\smallparagraph{Reducing the search space} We use the Pareto oracle to exclude sections of the search space by maintaining a dominated set $\dset$ and infeasible set $\iset$, that respectively contain points dominated by the current Pareto front and points guaranteed infeasible by a previous iteration (\cref{fig:ipro-stage1}). A naive approach would be to iteratively query the oracle and adjust $\dset$ and $\iset$ until they cover the entire bounding box. However, when Pareto oracle evaluations are expensive, such as when learning policies in a MOMDP, a more systematic approach is preferable to minimise the number of evaluations.

We propose selecting referents from a set of guaranteed lower bounds to maximise improvement in each iteration. Any remaining Pareto optimal solution $\vast$ must strictly dominate a point on the boundary of the dominated set $\dset$ and be upper bounded by a point on the boundary of the infeasible set $\iset$. Instead of considering the full boundaries, which contain infinitely many points, we restrict our attention to the inner corners. Formally, we define the lower bounds $\lset$ and upper bounds $\uset$, which cover these inner corners of $\dset$ and $\iset$ respectively (\cref{fig:ipro-stage1}). By this definition, $\vast$ dominates at least one $\vl \in \lset$, making it identifiable by a Pareto oracle. Moreover, since $\vast$ is dominated by some $\vu \in \uset$, $\uset$ provides guarantees on the distance to the remaining Pareto optimal solutions.

IPRO iteratively selects lower bounds from $\lset$ to query the oracle, updating boundaries based on the oracle's response. This process continues until the distance between every upper bound and its nearest lower bound falls below a user-defined tolerance $\tau$, ensuring a $\tau$-Pareto front is obtained. In practice, IPRO prioritises lower bounds using a heuristic selection function based on the hypervolume improvement metric, accelerating convergence by focusing on regions with the highest potential for exploration.

\smallparagraph{IPRO-2D} While in \cref{fig:ipro-sets} all unexplored sections are contained in isolated rectangles, this is a special property of bi-objective problems. In general, feasible solutions may dominate multiple lower bounds, necessitating careful updates (see the $\textproc{update}$ function in \cref{alg:ipro}). This allows for a dedicated variant, IPRO-2D, where significant simplifications can be made. When a new Pareto optimal solution is found, updating $\lset$ and $\uset$ requires adding at most two new points on either side of the adjusted boundary. Moreover, calculating the area of each rectangle is straightforward, enabling the construction of a priority queue which processes larger rectangles first to reduce the upper bound of the error quickly. Finally, instead of a full max-min operation required for the stopping criterion, the maximum error is determined by the rectangle with the greatest distance between its lower and upper bound.

\subsection{Upper bounding the error}
We now turn again to the general case for $d \geq 2$ objectives, demonstrating that $\uset$ may be used to bound the distance between the current approximation of the Pareto front $\vset_t$ and the remaining Pareto optimal solutions $\truepf \setminus \vset_t$. The true approximation error at timestep $t$ from $\vset_t$ to the true Pareto front $\truepf$ is defined as,
\begin{equation}
\label{eq:true-error}
\varepsilon^\ast_t = \sup_{\vast \in \truepf} \min_{\vv \in \vset} \|\vast - \vv\|_\infty.
\end{equation} 
Since $\uset$ is finite for any $t < \infty$ by construction, we can substitute the $\sup_{\vast \in \truepf}$ by a $\max_{\vu \in \uset}$, resulting in an upper bound on the true approximation error $\varepsilon^\ast_t$. We formalise this in \cref{th:approximation-guarantee}.

\begin{theorem}
\label{th:approximation-guarantee}
Let $\vset^\ast$ be the true Pareto front, $\vset_t$ the approximate Pareto front obtained by IPRO and $\varepsilon^\ast_t$ the true approximation error at timestep $t$. Then the following inequality holds,
\begin{equation}
\label{eq:upper-bound}
    \varepsilon^\ast_t \leq \max_{\vu \in \uset_t} \min_{\vv \in \vset_t} \|\vu - \vv\|_\infty.
\end{equation}
\end{theorem}
One can verify this in \cref{fig:ipro-stage2} where $\uset = \{\vu_1, \vu_3, \vu_4\}$ contains the upper bounds on the remaining Pareto optimal solutions. Note that while approximate Pareto fronts are commonly computed with regard to the $\normmax$ norm, this result can be extended to other metrics.

\subsection{Convergence to a Pareto front}
As IPRO progresses, the sequence of errors generated by \cref{th:approximation-guarantee} can be shown to be monotonically decreasing and converges to zero. Intuitively, this can be observed in \cref{fig:ipro-stage2} where the retrieval of a new Pareto optimal solution reduces the distance to the upper bounds. Additionally, the closure of a section, illustrated in \cref{fig:ipro-stage3}, results in the removal of the upper point which subsequently reduces the remaining search space. Since IPRO terminates when the true approximation error is guaranteed to be at most equal to the tolerance $\tau$, this results in a $\tau$-Pareto front.

\begin{theorem}
\label{th:ipro-convergence}
Given a Pareto oracle $\oracle$ and tolerance $\tau > 0$, IPRO converges to a $\tau$-Pareto front in a finite number of iterations. For a Pareto oracle $\oracle$ with tolerance $\tau = 0$, IPRO converges almost surely to the exact Pareto front as $t \to \infty$.
\end{theorem}
\begin{proofsketch}
As a corollary to \cref{th:approximation-guarantee} we first show that the sequence of errors produced by IPRO is monotonically decreasing. For $\tau > 0$, this sequence is further proven to converge to zero in a finite number of iterations. Since IPRO stops when the approximation error is at most $\tau$, this results in a $\tau$-Pareto front. 

For $\tau = 0$, we demonstrate under mild assumptions that the sequence of errors almost surely has its infimum at zero. By the monotone convergence theorem, we can therefore guarantee that IPRO almost surely converges to the exact Pareto front.
\end{proofsketch}

Finally, we analyse the complexity of IPRO for $\tau > 0$. As shown in \cref{th:ipro-convergence}, IPRO is guaranteed to terminate in a finite number of iterations; however, this number could still be arbitrarily large depending on $\tau$ and the number of objectives $d$. Similar to related work, we find that IPRO exhibits polynomial complexity in $\tau$ but exponential complexity in $d$ \cite{papadimitriou2000approximability,chatterjee2006markov}.

\begin{theorem}
\label{th:ipro-complexity}
Given a Pareto oracle $\oracle$ and tolerance $\tau > 0$, let $\forall j \in [d], k_j = \ceil{\nicefrac{(v^\text{i}_j - v^\text{n}_j)}{\tau}}$. IPRO constructs a $\tau$-Pareto front in at most
\begin{equation}
    \prod_{j=1}^d k_j - \prod_{j=1}^d (k_j - 1)
\end{equation}
iterations which is a polynomial in $\tau$ but exponential in the number of objectives $d$.
\end{theorem}
\begin{proofsketch} 
This bound is derived by constructing a worst-case scenario for IPRO in the grid induced by $\nadir, \ideal$ and $\tau$. We show that the worst case arises when covering $d$ facets with Pareto optimal solutions. The resulting bound is obtained by calculating the original number of cells in the grid, $\prod_{j=1}^d k_j$, and subtracting the number of cells in the smaller grid that excludes the Pareto optimal facets, $\prod_{j=1}^d (k_j - 1)$. 
\end{proofsketch}

\subsection{Dealing with imperfect Pareto oracles}
\label{sec:imperfect-oracles}
While IPRO relies on a Pareto oracle that solves the scalar problem exactly, this condition cannot always be guaranteed in practice when dealing with function approximators or heuristic solvers. To overcome this, we introduce a backtracking procedure that maintains the sequence $\left\{\left(\vl_t, \vv_t\right)\right \}_{t \in \mathbb{N}}$ of lower bounds and retrieved solution in each iteration. When, at iteration $n$, the returned solution $\vv_{n}$ strictly dominates a point $\vc \in \cset_n$ or $\vast \in \vset_n$, it indicates an incorrect oracle evaluation in a previous iteration and we initiate a replay of the sequence.

Let $\bar{t}$ represent the time step when the incorrect result was returned. For the subsequence $\left\{\left(\vl_t, \vv_t\right)\right \}_{0 \leq t < \bar{t}}$, we replay the pairs using standard IPRO updates and treat $\vv_{n}$ as the solution retrieved for $\vl_{\bar{t}}$. For the subsequent pairs $\left\{\left(\vl_t, \vv_t\right)\right \}_{\bar{t} < t <  n}$, we verify for each $(\vv_t, \vl_t)$ whether the original evaluation succeeded. If so, $\vv_t$ was weakly Pareto optimal, and if a new lower bound $\vl'$ exists that is dominated by $\vv_t$, we perform an update with $(\vl', \vv_t)$. If the evaluation failed, $\vl_t$ was marked as complete, and we check whether a new lower bound $\vl'$ dominates $\vl_t$. If so, $\vl'$ is also marked as complete. This mechanism corrects earlier mistakes and reuses previous iteration outcomes as efficiently as possible.


\section{Pareto oracle}
\label{sec:pareto-oracle}
Obtaining a solution in a designated region is central to IPRO. We introduce Pareto oracles for this purpose and derive theoretically sound methods that lead to effective implementations in practice.

\subsection{Formalisation} 
\label{sec:po-definition}
In each iteration, IPRO queries a Pareto oracle with a referent from the lower bounds to identify a new weakly Pareto optimal policy in the target region. We define two Pareto oracle variants that differ in the quality of the returned policy and their adherence to the target region. When zero tolerance is required, a \emph{weak} Pareto oracle returns weakly Pareto optimal solutions.

\begin{definition}
\label{def:weak-pareto-oracle}
A weak Pareto oracle $\oracle$ with tolerance $\tau = 0$ maps a referent $\vr \in \mathbb{R}^d$ to a weakly Pareto optimal policy $\policy \in \policies$ such that $\vpi > \vr$ or returns $\textsc{False}$ when no such policy exists.
\end{definition}

While \cref{def:weak-pareto-oracle} requires no tolerance, limiting solutions to weakly Pareto optimal ones may be restrictive in practice. To address this, we define \emph{approximate} Pareto oracles, which guarantee Pareto optimal solutions but require a strictly positive tolerance. This ensures that each iteration yields meaningful progress—either identifying a new Pareto optimal solution with at least minimal improvement over the lower bound or closing an entire section. Since these oracles return Pareto optimal rather than merely weakly optimal solutions, fewer evaluations are required overall.

\begin{definition}
\label{def:approximate-pareto-oracle}
An approximate Pareto oracle $\oracle$ with intrinsic tolerance $\bar\tau \geq 0$ and user-provided tolerance $\tau > \bar\tau$ maps a referent $\vr \in \mathbb{R}^d$ to a Pareto optimal policy $\policy \in \policies$ such that $\vpi \pde \vr + \tau$ or returns $\textsc{False}$ when no such policy exists.
\end{definition}

Unlike weak Pareto oracles, approximate Pareto oracles incorporate an \emph{intrinsic} tolerance $\bar\tau$ alongside a user-defined tolerance $\tau$. Intuitively, $\bar\tau$ represents the minimal adjustment needed to ensure the oracle returns solutions strictly within the target region. The user-defined tolerance, being strictly greater, determines the minimal improvement necessary to justify further exploration. In some implementations, $\bar\tau$ is zero, allowing the user to freely select any tolerance (see \cref{sec:principled-implementations}).

To illustrate the difference between a weak and approximate Pareto oracle, we show a possible evaluation of both oracles in \cref{fig:weak-approx-po}. We note that related concepts have been studied in multi-objective optimisation \citep{papadimitriou2000approximability} and planning \citep{chatterjee2006markov}.

\begin{figure}[tb]
    \centering
        \begin{subfigure}[b]{0.34\textwidth}
        \centering
        \includegraphics[trim={0cm 0cm 0cm 0cm},clip,width=\textwidth]{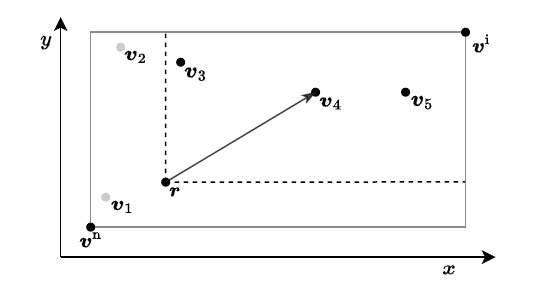}
        \subcaption{A weak Pareto oracle.}
        \label{fig:weak-po}
    \end{subfigure}
    \begin{subfigure}[b]{0.34\textwidth}
        \centering
        \includegraphics[trim={0cm 0cm 0cm 0cm},clip,width=\textwidth]{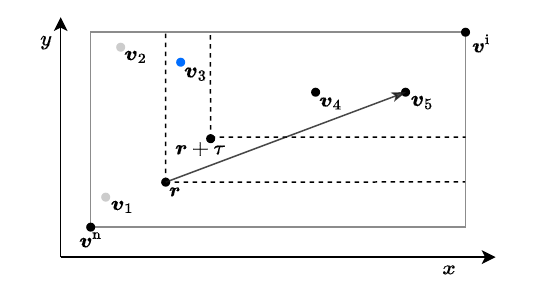}
        \subcaption{An approximate Pareto oracle.}
        \label{fig:approx-po}
    \end{subfigure}
    \Description{The figure illustrates two types of Pareto oracles: a weak Pareto oracle (top) and an approximate Pareto oracle (bottom).}
    \caption{Solutions inside the target region are black, while solutions outside the target region are grey. (a) The weak Pareto oracle returns $\vv_4$, which is in the target region but is only weakly Pareto optimal as it is dominated by $\vv_5$. (b) The approximate Pareto oracle returns a Pareto optimal solution $\vv_5$, but cannot find $\vv_3$, shown in blue.} 
    \label{fig:weak-approx-po}
\end{figure}

\subsection{Relation to achievement scalarising functions}
\label{sec:relation-asf}
In \cref{sec:preliminaries}, we introduced order representing and order approximating achievement scalarising functions (ASFs) and their role in obtaining (weakly) Pareto optimal solutions. Here, we demonstrate their direct application in constructing Pareto oracles.

We first show that evaluating a weak Pareto oracle $\oracle$ can be framed as maximising an order representing ASF over a set of allowed policies $\policies$. Since such ASFs guarantee that their maximum is reached within the target region at some weakly optimal solution, \cref{th:weak-po} follows immediately.

\begin{theorem}
\label{th:weak-po}
Let $\sr$ be an order representing ASF. Then $\oracle(\vr) = \argmax_{\policy \in \policies} \sr(\vpi)$ with tolerance $\tau = 0$ is a valid weak Pareto oracle.
\end{theorem}

This ensures that weakly optimal solutions can be obtained by proposing referents to an order representing ASF. However, practical considerations may lead us to favour an order approximating ASF, which yields Pareto optimal solutions instead. We demonstrate in \cref{th:approx-po} that such ASFs can indeed be applied to construct approximate Pareto oracles.

\begin{theorem}
\label{th:approx-po}
Let $\sr$ be an order approximating ASF and let $\vl \in \mathbb{R}^d$ be a lower bound such that only referents $\vr$ are selected when $\vr \pde \vl$. Then $\sr$ has an inherent oracle tolerance $\bar\tau > 0$ and for any user-provided tolerance $\tau > \bar\tau$, $\oracle(\vr) = \argmax_{\policy \in \policies} s_{\vr + \tau}(\vpi)$ is a valid approximate Pareto oracle.
\end{theorem}

By definition, an order approximating ASF attains its maximum at a Pareto optimal solution. However, since such ASFs assign non-negative values to solutions outside the target region, this maximum may occur outside the desired area. To mitigate this, we introduced the inherent tolerance in \cref{def:approximate-pareto-oracle}. Ensuring $\tau > \bar\tau$ guarantees that new solutions remain within the correct region. Since directly determining $\bar\tau$ can be challenging, a practical alternative is to use an order approximating ASF while still optimising $\argmax_{\vpi \in \policies} \sr(\vpi)$, as done in the weak Pareto oracle.

\subsection{Principled implementations}
\label{sec:principled-implementations}
While \cref{th:weak-po,th:approx-po} establish that Pareto oracles may be implemented using an ASF, optimising the ASF over a given policy class may still be challenging. Here, we show that efficient implementations can be derived from existing literature. First, the proposed approach using ASFs can be implemented by solving an auxiliary \emph{convex} MDP in which the goal is to minimise a convex function over a set of admissible stationary distributions \citep{zahavy2021reward}. Recent work has proposed multiple methods that come with strong convergence guarantees to solve convex MDPs \citep{zhang2020variational,zahavy2021reward,geist2022concave}. 
\begin{proposition}
\label{prop:convex-mdp-po}
Let $\sr$ be an ASF that is concave for any $\vr \in \mathbb{R}^d$. Then, for any $\vr \in \mathbb{R}^d$ and tolerance $\tau \geq 0$, a valid weak or approximate Pareto oracle $\oracle$ can be implemented for the class of stochastic policies by solving an auxiliary convex MDP. 
\end{proposition}

In addition, approximate Pareto oracles can be implemented without optimising an ASF but rather by solving an auxiliary \emph{constrained} MDP. Treating the referent as lower bound constraints and maximising the sum of rewards can be shown to result in a Pareto optimal solution inside the target region if one exists. One important advantage of this oracle is that there is no inherent tolerance and so $\tau$ can be chosen freely.

\begin{proposition}
\label{prop:constrained-mdp-po}
For any referent $\vr \in \mathbb{R}^d$, tolerance $\tau > 0$ and policy class, a valid approximate Pareto oracle $\oracle$ can be implemented by solving an auxiliary constrained MDP.
\end{proposition}
Several algorithms with strong theoretical foundations have been proposed for solving such models in a reinforcement learning context \citep{achiam2017constrained,ding2021provably}. When the constrained MDP is known and the state and action sets are finite, an optimal stochastic policy can be computed in polynomial time \citep{altman1999constrained}. Together with \cref{th:ipro-complexity}, this guarantees that IPRO obtains a Pareto front of stochastic policies in polynomial time, recovering prior guarantees \citep{papadimitriou2000approximability,chatterjee2006markov}. Although computing optimal stationary deterministic policies in constrained MDPs is NP-complete \citep{feinberg2000constrained}, mixed-integer linear programming has been shown to be effective in practice \citep{dolgov2005stationary}.

\section{Deterministic memory-based policies}
\label{sec:dmb-policies}
As shown in \cref{sec:ipro,sec:pareto-oracle}, IPRO obtains the Pareto front for any policy class with a valid Pareto oracle. We now develop a Pareto oracle specifically for deterministic memory-based policies, a class for which there is currently no method that can learn non-convex Pareto fronts in general MOMDPs.

\subsection{Motivation}
In single-objective MDPs, an optimal deterministic policy is always guaranteed to exist. However, in MOMDPs, this result does not hold, and stochastic policies may be required to capture all solutions on the Pareto front. Nevertheless, in practical applications where interpretability, explainability, and safety are critical, deterministic policies remain preferable, as noted in related work \cite{hayes2022practical}. For example, in medical applications, decisions must be interpretable, with deterministic treatment protocols being essential.

To avoid the need for randomisation in policies, memory can be used to learn additional policies that provide alternative trade-offs for the decision-maker. Consider a pick-up and delivery MOMDP where the agent can either collect a package (yielding a reward of $(3, 0)$) or deliver it (yielding $(0, 3)$), with both actions returning to the same state. Without memory, deterministic policies are restricted to always collecting or always delivering, resulting in discounted returns of $(\nicefrac{3}{1 - \gamma}, 0)$ or $(0, \nicefrac{3}{1 - \gamma})$. By incorporating memory, the agent can condition its actions on past behaviour—for instance, delivering after each collection—achieving a discounted return of $(\nicefrac{3}{1 - \gamma^2}, \nicefrac{3\gamma}{1 - \gamma^2})$. This demonstrates how memory increases the set of feasible Pareto optimal policies, as proved by \citet{white1982multiobjective}.

\subsection{ASF selection}
In our experimental evaluation, we utilise the well-known augmented Chebyshev scalarisation function \citep{nikulin2012new}, shown in \cref{eq:aasf}. We highlight that this ASF is concave for all referents, implying its applicability together with \cref{prop:convex-mdp-po} for stochastic policies as well.
\begin{equation}
\label{eq:aasf}
\sr(\vv) = \min_{j \in \{1, \dotsc, d\}} \lambda_j (v_j - r_j) + \rho \sum_{j=1}^d \lambda_j (v_j - r_j)
\end{equation}
Here, $\vect{\lambda} > 0$ serves as a normalisation constant for the different objectives, and $\rho$ is a parameter determining the strength of the augmentation term. Selecting $\vect{\lambda} = (\ideal - \nadir)^{-1}$ scales any vector $\vv$ relative to the distance between the nadir $\nadir$ and ideal $\ideal$, thereby ensuring a balanced scale across all objectives. This normalisation prevents the dominance of one objective over another, a challenge that is otherwise difficult to overcome \citep{abdolmaleki2020distributional}.

\Cref{eq:aasf} serves as a weak or approximate Pareto oracle, depending on the augmentation parameter $\rho$. When $\rho = 0$, the augmentation term is cancelled and the minimum ensures that only vectors in the target region have non-negative values. However, optimising a minimum may result in weakly Pareto optimal solutions (e.g. $(1, 2)$ and $(1, 1)$ share the same minimum). For $\rho > 0$, the optimal solution will be Pareto optimal (the sum of $(1, 2)$ is greater than that of $(1, 1)$) but may exceed the target region.

\subsection{Practical implementation}
\label{sec:practical-implementation}
In \cref{sec:principled-implementations} we demonstrated that Pareto oracle implementations with strong guarantees exist for stochastic policies. In contrast, obtaining a Pareto optimal policy that dominates a given referent is NP-hard for memory-based deterministic policies \citep{chatterjee2006markov}. To address this, we extend single-objective reinforcement learning algorithms to optimise the ASF in \cref{eq:aasf}. It is common in MORL to encode the memory of a policy using its accrued reward at timestep $t$ defined as $\accrewards{t} := \sum_{k=0}^{t-1} \gamma^{k} \vecrewards(\s_{k}, \action_{k}, \s_{k+1})$. In our implementation, this accrued reward is directly added to the observation.

\smallparagraph{DQN} We extend the GGF-DQN algorithm, which optimises for the generalised Gini welfare of the expected returns \citep{siddique2020learning}, to optimise any scalarisation function $f$. We note that GGF-DQN is itself an extension of DQN \citep{mnih2015humanlevel}. Concretely, we train a Q-network such that $\vect{Q}(\s_t, \action_t) = \vr + \gamma \vect{Q}(\s_{t+1}, \action^\ast)$ where the optimal action $\action^\ast$ is computed using the accrued reward and scalarisation function $f$,
\begin{equation}
    \action^\ast = \argmax_{\action \in \actions} f \left(\accrewards{t+1} + \gamma \vect{Q} \left(\s_{t+1}, \action \right) \right).
\end{equation}
One limitation of this action selection method is that it does not perfectly align with the objective to be optimised since,
\begin{equation}
    f\fun{\vpi} = f\fun{\expected{\policy, \sinit}{\accrewards{t+1}} + \gamma \vect{Q} \left(\s_{t+1}, \action \right)}.
\end{equation}
As computing the expectation of $\accrewards{t+1}$ is usually impractical, we use the observed accrued reward as a substitute.

\smallparagraph{Policy gradient} We extend A2C \citep{mnih2016asynchronous} and PPO \citep{schulman2017proximal} to optimise $J(\policy) = f(\vv^\policy)$, where $f$ is a scalarisation function and $\policy$ a parameterised policy with parameters $\theta$. For differentiable $f$, the policy gradient becomes $\nabla_\theta J(\policy) = f'(\vv^{\policy}) \cdot \nabla_\theta \vv^{\policy}(s_0)$ \citep{reymond2023actorcritic}. 
To ensure deterministic policies, we take actions according to $\argmax_{a \in \mathcal{A}}\policy(a|s)$ during policy evaluation. Although this potentially changes the policy, effectively employing a policy that differs from the one initially learned, empirical observations suggest that these algorithms typically converge toward deterministic policies in practice. Furthermore, recent work has theoretically analysed this practice and found that under some assumptions convergence to the optimal deterministic policy is guaranteed \citep{montenegro2024learning}.

\smallparagraph{Extended networks}
Rather than making separate calls to one of the previous reinforcement learning methods for each oracle evaluation, we employ extended networks \citep{abels2019dynamic} to improve sample efficiency. Concretely, we extend our actor and critic networks to take a referent as additional input, enabling their reuse across IPRO iterations. We further introduce a pre-training phase, where a policy is trained on randomly sampled referents for a fixed number of episodes. To maximise the benefit of this pre-training, we perform additional off-policy updates for referents not used in data collection. While this has no effect on DQN, policy gradient methods require alignment between behaviour and target policies. We address this through importance sampling in A2C and an off-policy adaptation of PPO \citep{meng2023offpolicy}.

\section{Experiments}
\label{sec:experiments}
To test the performance of IPRO, we combine it with the modified versions of DQN, A2C, and PPO proposed in \cref{sec:dmb-policies} as approximate Pareto oracles that optimise the augmented Chebyshev scalarisation function in \cref{eq:aasf}. All experiments are repeated over five seeds and additional details are 
presented in \cref{ap:experiment-details}. Our code is available at \url{https://github.com/wilrop/ipro}.

\subsection{Evaluation metrics} 
Evaluating MORL algorithms poses significant challenges due to the difficulty in measuring the quality of a Pareto front \citep{felten2023toolkit}. To address this, we compute two different metrics during learning and one for the final returned front.

We first consider the hypervolume, defined in \cref{eq:hv}, a well-established measure in MORL. The hypervolume quantifies the volume of the dominated region formed by the current estimate of the Pareto front relative to a specified reference point. However, a notable drawback of this metric is that the choice of reference point significantly influences the obtained values, potentially distorting the results. In practice, we use the nadir as the reference point.
\begin{equation}
\label{eq:hv}
    HV(\vset_t; \vr) = \volume \fun{\bigcup_{\vv \in \vset_t} \left[\vr, \vv \right]}
\end{equation}
Following the approach outlined by \citet{hayes2022practical}, we further evaluate all algorithms using utility-based metrics. Concretely, for a solution set $\vset_t$ at timestep $t$ we compute the maximum utility loss (MUL) \citep{zintgraf2015quality} compared to the true Pareto front $\truepf$ as
\begin{equation}
\label{eq:mul}
    MUL(\vset_t; \truepf) = \max_{u \in U}\left[\max_{\vv \in \truepf} u(\vv) - \max_{\vv \in \vset_t} u(\vv)\right].
\end{equation} 
We generate piecewise linear, monotonically increasing functions $u: [\nadir, \ideal] \to [0, 1]$ by sampling a grid of positive numbers as gradients. The function value at $\vv$ is obtained by summing the preceding gradients and rescaling. Our grid uses six cells per dimension, with gradients drawn uniformly from $[0, 5)$. Notably, this method produces functions biased towards risk aversion. Furthermore, we estimate $\truepf$ as the union of all final Pareto fronts obtained by both IPRO and the baseline algorithms across all runs. Lastly, we evaluate the quality of the final Pareto front by its true error as defined in \cref{eq:true-error}. This metric provides an additional measure of how closely the final approximation aligns with the true Pareto front.

\subsection{Baselines}
As IPRO is the first general-purpose method capable of learning the Pareto front for arbitrary policies in general MOMDPs, we select baselines that are tailored to specific settings. To ensure a fair comparison, we extend all baselines to accumulate their empirical Pareto fronts across evaluation steps, guaranteeing the same monotonic improvement as in IPRO.

\smallparagraph{Convex hull algorithms}
We evaluate two state-of-the-art convex hull algorithms: Generalised Policy Improvement - Linear Support (GPI-LS) \citep{alegre2023sampleefficient} and Envelope Q-Learning (EQL) \citep{yang2019generalized}. Both algorithms train vectorial Q-networks that can be dynamically adjusted with given weights to produce a scalar return. 

\smallparagraph{Pareto front algorithm}
We include Pareto Conditioned Networks (PCN), which were specifically designed to learn the Pareto front of deterministic policies in deterministic MOMDPs \citep{reymond2022pareto}. PCN trains a network to generalise across the full Pareto front by predicting the ``return-to-go'' for any state and selecting the action that best aligns with the desired trade-off.

\begin{figure*}[tb]
    \centering
    \begin{subfigure}[b]{0.32\textwidth}
        \centering
        \includegraphics[width=\textwidth]{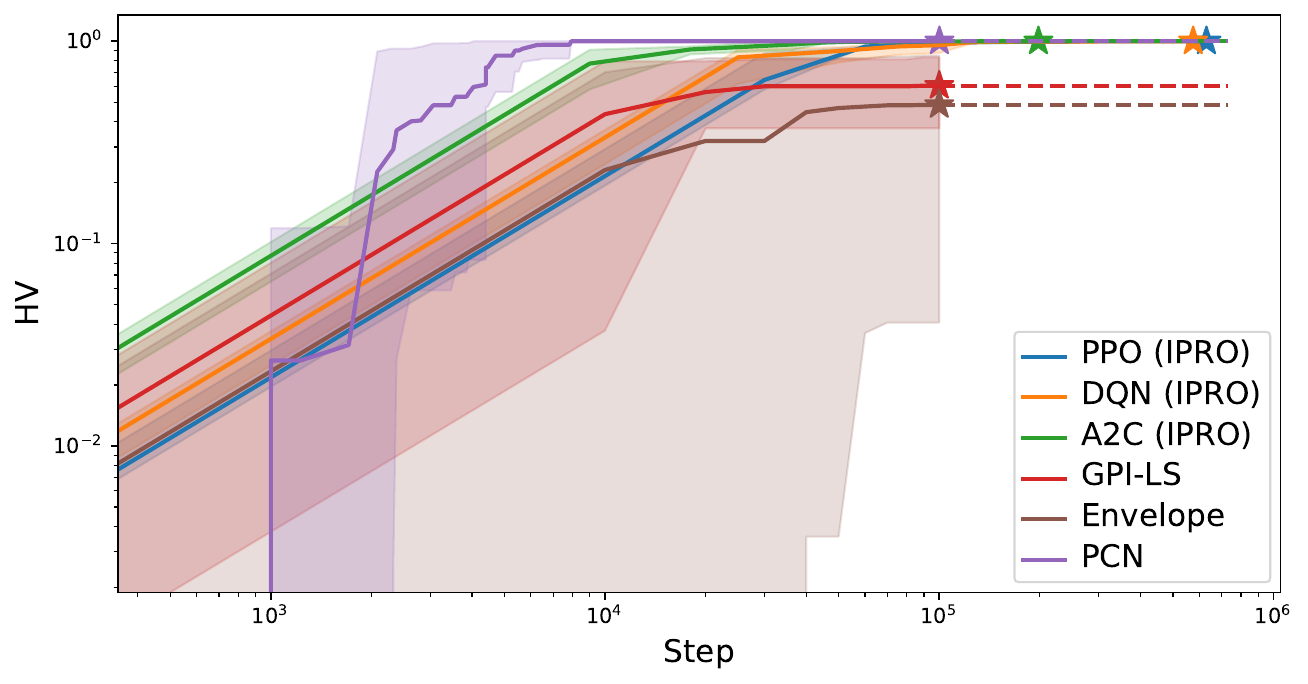}
        \label{fig:hv-dst}
    \end{subfigure}
    \begin{subfigure}[b]{0.32\textwidth}
        \centering
        \includegraphics[width=\textwidth]{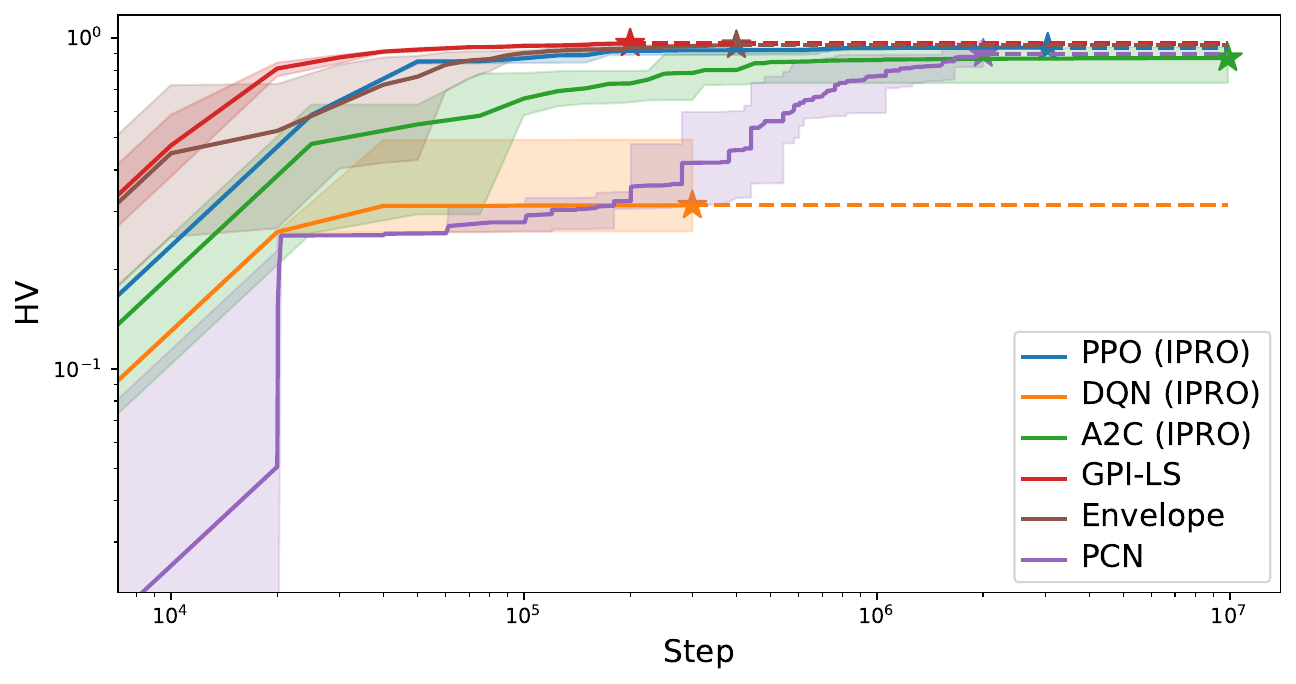}
        \label{fig:hv-minecart}
    \end{subfigure}
    \begin{subfigure}[b]{0.32\textwidth}
        \centering
        \includegraphics[width=\textwidth]{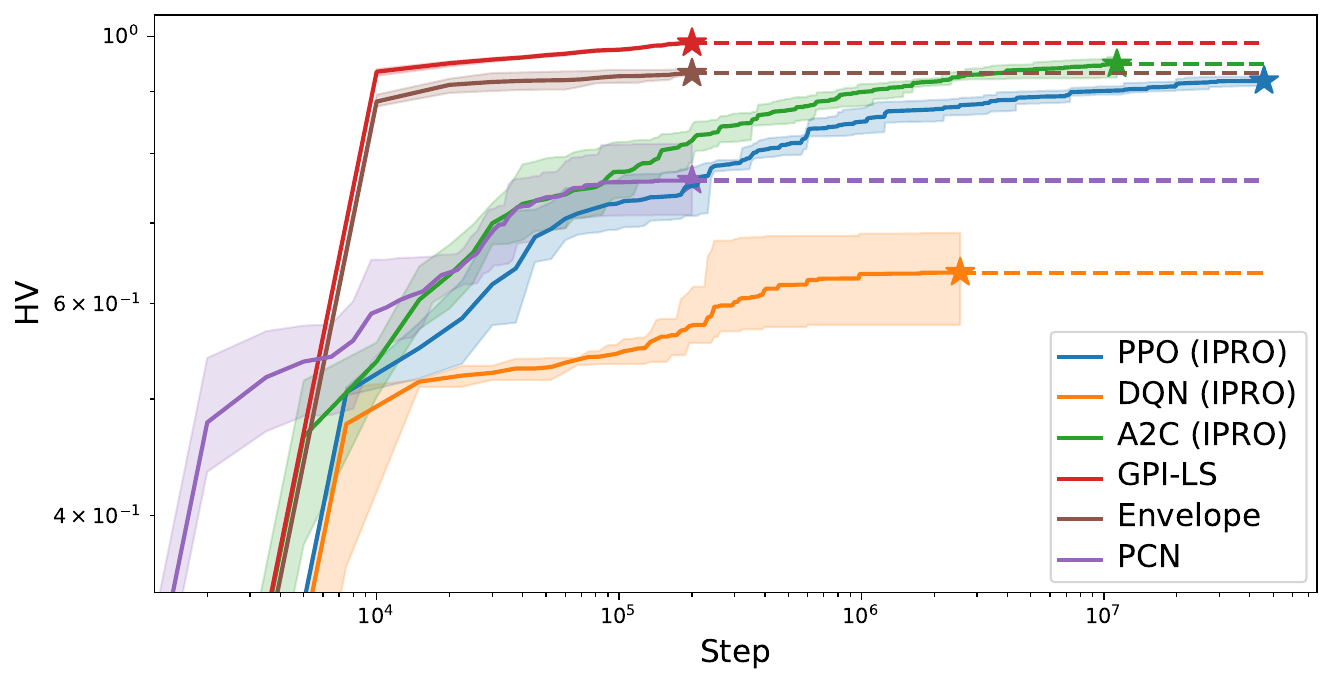}
        \label{fig:hv-reacher}
    \end{subfigure}
    \begin{subfigure}[b]{0.32\textwidth}
        \centering
        \includegraphics[width=\textwidth]{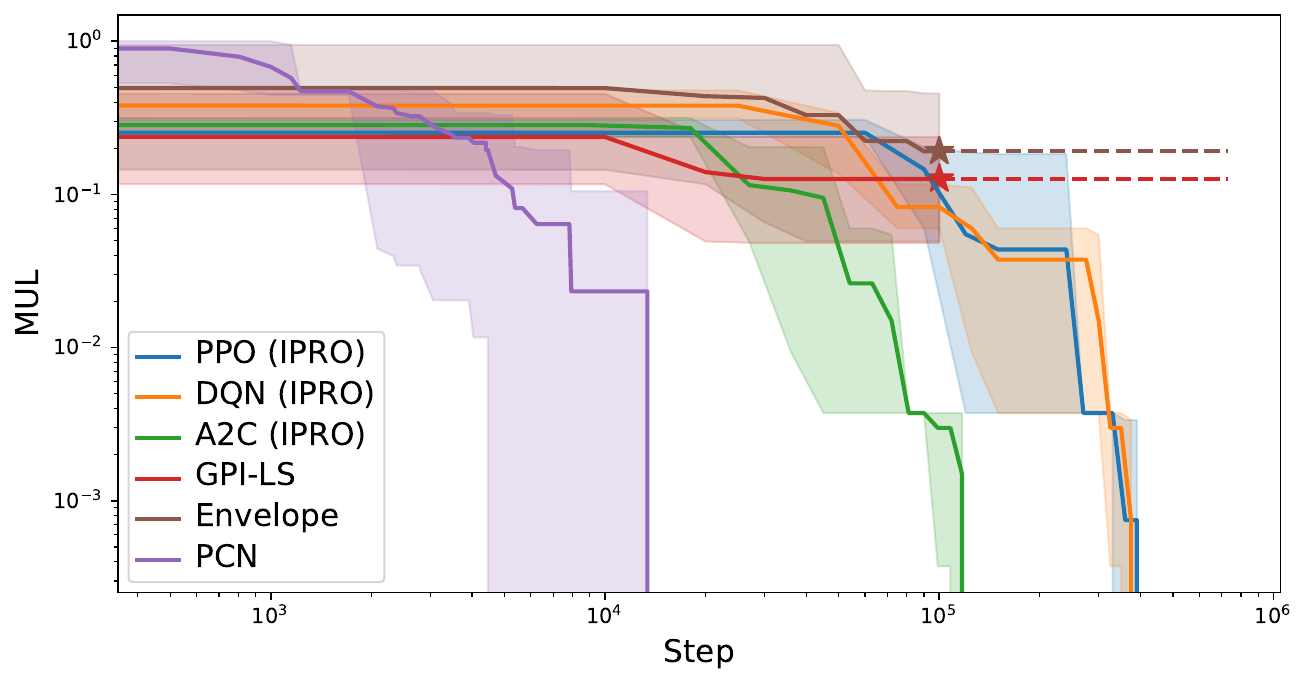}
        \subcaption{Deep Sea Treasure}
        \label{fig:mul-dst}
    \end{subfigure}
    \begin{subfigure}[b]{0.32\textwidth}
        \centering
        \includegraphics[width=\textwidth]{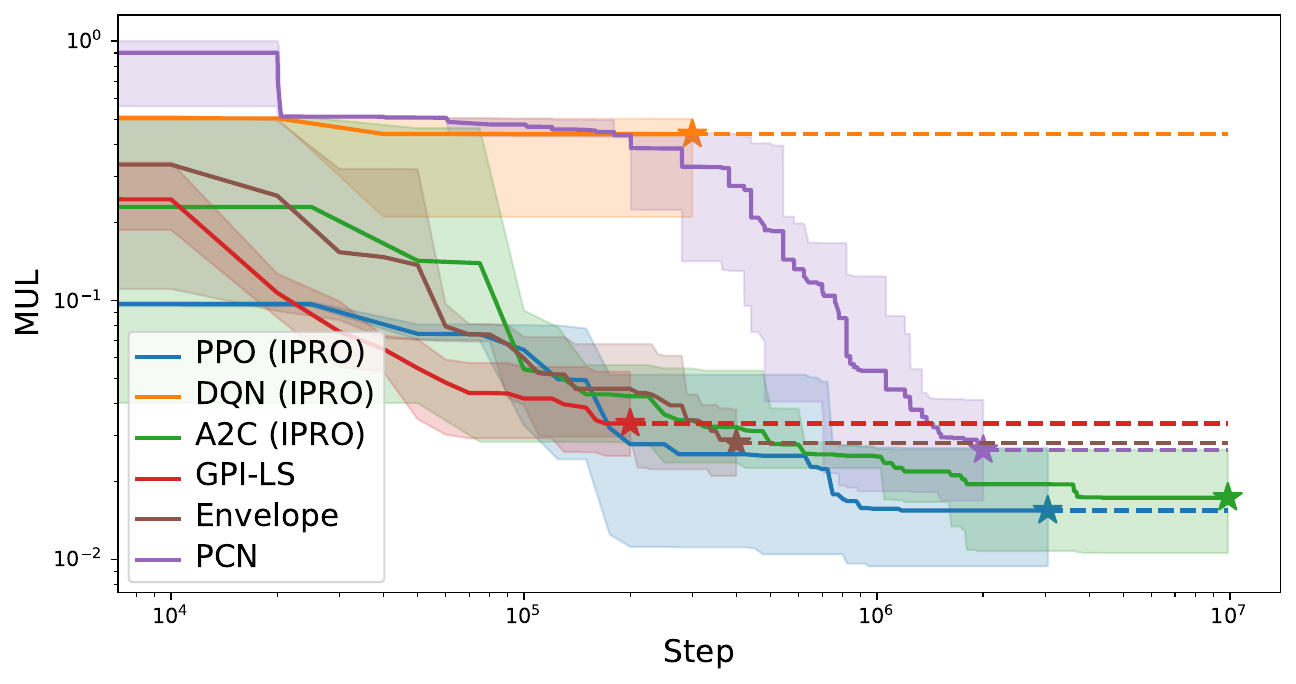}
        \subcaption{Minecart}
        \label{fig:mul-minecart}
    \end{subfigure}
    \begin{subfigure}[b]{0.32\textwidth}
        \centering
        \includegraphics[trim={0cm 0cm 0cm 0cm},clip,width=\textwidth]{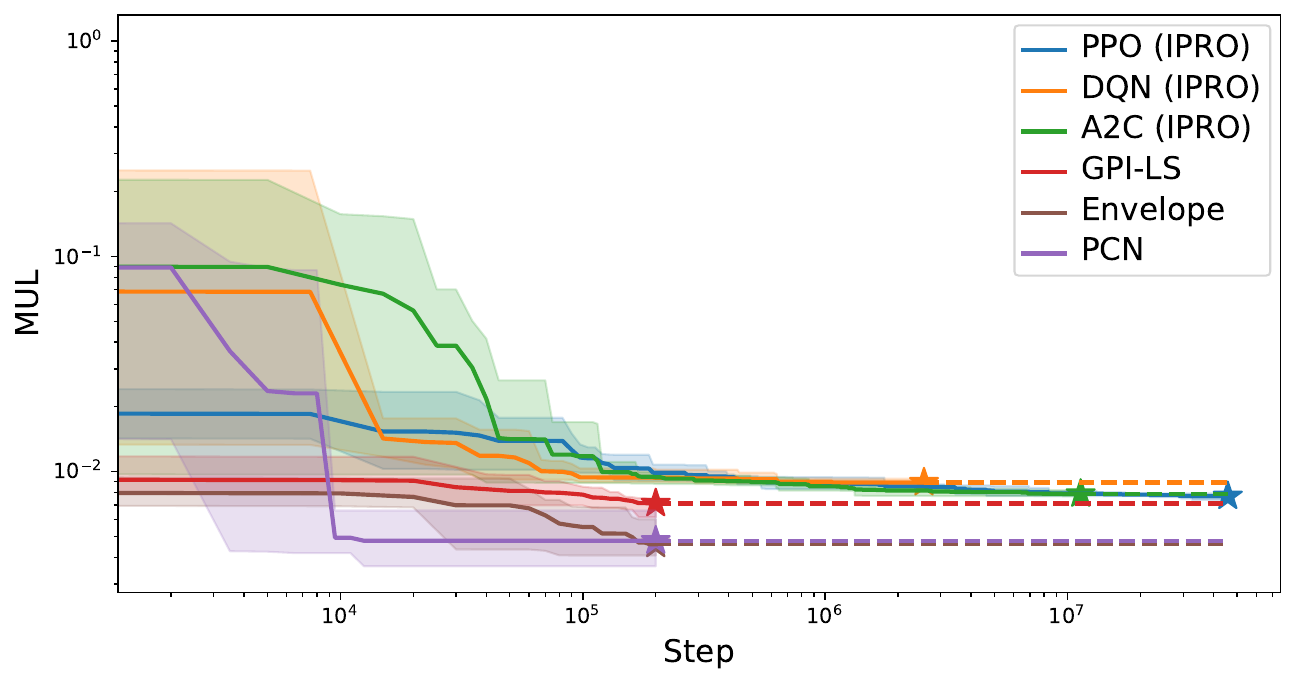}
        \subcaption{MO-Reacher}
        \label{fig:mul-reacher}
    \end{subfigure}
    \Description{}
    \caption{The mean hypervolume (top) and maximum utility loss (bottom) scaled between zero and one with 95-percentile interval on a log-log scale. Stars indicate when each algorithm finishes. The pretraining phase of IPRO is not shown.}
    \label{fig:experiments}
\end{figure*}

\subsection{Results}

\smallparagraph{Deep Sea Treasure $(d=2)$} Deep Sea Treasure (DST) is a deterministic environment where a submarine seeks treasure while minimising fuel consumption. DST has a Pareto front with solutions in concave regions \citep{vamplew2011empirical}, making it impossible for the convex hull algorithms to recover all Pareto optimal solutions. This limitation is evident in \cref{fig:hv-dst,fig:mul-dst} where GPI-LS and EQL exhibit significantly inferior performance compared to IPRO and PCN. Notably, IPRO and PCN recover the complete Pareto front in the majority of runs; however, IPRO tends to require more samples. This discrepancy can be attributed to the fact that IPRO learns only one Pareto optimal solution per iteration, whereas PCN concurrently learns multiple policies. Nonetheless, this concurrent learning approach for PCN comes at the expense of theoretical guarantees. When comparing the $\varepsilon$ metric (\cref{tab:approx-quality}), we observe that IPRO learns high-quality approximations and consistently learns the complete Pareto front when paired with PPO and DQN. The convex hull methods naturally have poorer approximations.

\begin{table}[t]
\caption{The minimum $\varepsilon$ shift necessary to obtain any undiscovered Pareto optimal solution.}
\label{tab:approx-quality}
\centering
\begin{small}
\begin{sc}
\begin{tabular}{llc}
\toprule
Env & Algorithm & $\varepsilon$ \\
\midrule
& IPRO (PPO)    & $\bm{0.0} \pm 0.0$ \\
& IPRO (A2C)    & $0.2 \pm 0.4$ \\
DST & IPRO (DQN)    & $\bm{0.0} \pm 0.0$ \\
& PCN           & $\bm{0.0} \pm 0.0$ \\
& GPI-LS        & $5.2 \pm 2.71$ \\
& Envelope      & $28.6 \pm 46.77$ \\
\midrule
& IPRO (PPO)    & $0.66 \pm 0.07$ \\
& IPRO (A2C)    & $0.54 \pm 0.11$ \\
Minecart & IPRO (DQN)    & $1.11 \pm 0.01$ \\
& PCN           & $0.67 \pm 0.2$ \\
& GPI-LS        & $\bm{0.42} \pm 0.0$ \\
& Envelope      & $\bm{0.42} \pm 0.01$ \\
\midrule
& IPRO (PPO)    & $5.75 \pm 1.22$ \\
& IPRO (A2C)    & $\bm{2.84} \pm 0.39$ \\
MO-Reacher & IPRO (DQN)    & $15.02 \pm 1.42$ \\
& PCN           & $18.95 \pm 1.76$ \\
& GPI-LS        & $8.5 \pm 0.12$ \\
& Envelope      & $11.41 \pm 0.62$ \\
\bottomrule
\end{tabular}
\end{sc}
\end{small}
\end{table}

\smallparagraph{Minecart $(d=3)$} Minecart is a stochastic environment where the agent collects two types of ore while minimising fuel consumption \citep{abels2019dynamic}. Since this environment was designed to induce a convex Pareto front, GPI-LS and EQL are expected to perform optimally. We find that IPRO achieves comparable hypervolume results and demonstrates superior maximum utility loss (MUL) compared to all other baselines when using policy gradient oracles. The anytime property of IPRO is particularly evident in the MUL results, as its Pareto front continues to improve up to $10^7$ steps. In \cref{tab:approx-quality}, the $\varepsilon$ distances for the policy gradient methods are shown to be competitive. However, we observe that the DQN variant struggles to learn a qualitative Pareto front, which may be attributed to the algorithm's ad-hoc nature. This suggests that future research focusing on value-based oracles could provide significant benefits.

\smallparagraph{MO-Reacher $(d=4)$} MO-Reacher is a deterministic environment where four balls are arranged in a circle and the goal is to minimise the distance to each ball. Since it is deterministic and has a mostly convex Pareto front, it suits all baselines. In \cref{fig:mul-reacher}, we find that IPRO obtains a hypervolume and maximum utility loss competitive to the baselines. Additionally, the policy gradient oracles result in the best approximations to the Pareto front according to the $\varepsilon$ metric in \cref{tab:approx-quality}. Due to IPRO's iterative mechanism, this comes at the price of increased sample complexity, while the baselines benefit from learning multiple policies concurrently.


These results demonstrate IPRO's competitiveness to the baselines in all environments, an impressive feat given that all baselines perform significantly worse in one of the environments. Moreover, IPRO stands out without requiring domain knowledge for proper application, unlike its competitors.

\section{Conclusion}
We introduce IPRO to provably learn a Pareto front in MOMDPs. IPRO iteratively proposes referents to a Pareto oracle and uses the returned solution to trim sections from the search space. We formally define Pareto oracles and derive principled implementations. We show that IPRO converges to a Pareto front and comes with strong guarantees with respect to the approximation error. Our empirical analysis finds that IPRO learns high-quality Pareto fronts while requiring less domain knowledge than baselines. For future work, we aim to extend IPRO to learn multiple policies concurrently and explore alternative Pareto oracle implementations.

\begin{acks}
We thank Conor Hayes and Enda Howley for their guidance throughout various stages of this work.
WR is supported by the Research Foundation – Flanders (FWO), grant number 1197622N. MR received support through Prof. Irina Rish. This research was supported by funding from the Flemish Government under the ``Onderzoeksprogramma Artifici\"{e}le Intelligentie (AI) Vlaanderen'' program.
\end{acks}



\bibliographystyle{ACM-Reference-Format} 
\balance
\bibliography{bibliography}


\begin{thebibliography}{39}


\ifx \showCODEN    \undefined \def \showCODEN     #1{\unskip}     \fi
\ifx \showDOI      \undefined \def \showDOI       #1{#1}\fi
\ifx \showISBNx    \undefined \def \showISBNx     #1{\unskip}     \fi
\ifx \showISBNxiii \undefined \def \showISBNxiii  #1{\unskip}     \fi
\ifx \showISSN     \undefined \def \showISSN      #1{\unskip}     \fi
\ifx \showLCCN     \undefined \def \showLCCN      #1{\unskip}     \fi
\ifx \shownote     \undefined \def \shownote      #1{#1}          \fi
\ifx \showarticletitle \undefined \def \showarticletitle #1{#1}   \fi
\ifx \showURL      \undefined \def \showURL       {\relax}        \fi
\providecommand\bibfield[2]{#2}
\providecommand\bibinfo[2]{#2}
\providecommand\natexlab[1]{#1}
\providecommand\showeprint[2][]{arXiv:#2}

\bibitem[\protect\citeauthoryear{Abdolmaleki, Huang, Hasenclever, Neunert, Song, Zambelli, Martins, Heess, Hadsell, and Riedmiller}{Abdolmaleki et~al\mbox{.}}{2020}]%
        {abdolmaleki2020distributional}
\bibfield{author}{\bibinfo{person}{Abbas Abdolmaleki}, \bibinfo{person}{Sandy Huang}, \bibinfo{person}{Leonard Hasenclever}, \bibinfo{person}{Michael Neunert}, \bibinfo{person}{Francis Song}, \bibinfo{person}{Martina Zambelli}, \bibinfo{person}{Murilo Martins}, \bibinfo{person}{Nicolas Heess}, \bibinfo{person}{Raia Hadsell}, {and} \bibinfo{person}{Martin Riedmiller}.} \bibinfo{year}{2020}\natexlab{}.
\newblock \showarticletitle{A Distributional View on Multi-Objective Policy Optimization}. In \bibinfo{booktitle}{\emph{Proceedings of the 37th {{International Conference}} on {{Machine Learning}}}}, \bibfield{editor}{\bibinfo{person}{Hal~Daum{\'e} III} {and} \bibinfo{person}{Aarti Singh}} (Eds.), Vol.~\bibinfo{volume}{119}. \bibinfo{publisher}{{PMLR}}, \bibinfo{pages}{11--22}.
\newblock


\bibitem[\protect\citeauthoryear{Abels, Roijers, Lenaerts, Now{\'e}, and Steckelmacher}{Abels et~al\mbox{.}}{2019}]%
        {abels2019dynamic}
\bibfield{author}{\bibinfo{person}{Axel Abels}, \bibinfo{person}{Diederik~M. Roijers}, \bibinfo{person}{Tom Lenaerts}, \bibinfo{person}{Ann Now{\'e}}, {and} \bibinfo{person}{Denis Steckelmacher}.} \bibinfo{year}{2019}\natexlab{}.
\newblock \showarticletitle{Dynamic {{Weights}} in {{Multi-Objective Deep Reinforcement Learning}}}. In \bibinfo{booktitle}{\emph{Proceedings of the 36th {{International Conference}} on {{Machine Learning}}}}, \bibfield{editor}{\bibinfo{person}{Kamalika Chaudhuri} {and} \bibinfo{person}{Ruslan Salakhutdinov}} (Eds.), Vol.~\bibinfo{volume}{97}. \bibinfo{publisher}{{PMLR}}, \bibinfo{address}{{Long Beach, California, USA}}, \bibinfo{pages}{11--20}.
\newblock


\bibitem[\protect\citeauthoryear{Achiam, Held, Tamar, and Abbeel}{Achiam et~al\mbox{.}}{2017}]%
        {achiam2017constrained}
\bibfield{author}{\bibinfo{person}{Joshua Achiam}, \bibinfo{person}{David Held}, \bibinfo{person}{Aviv Tamar}, {and} \bibinfo{person}{Pieter Abbeel}.} \bibinfo{year}{2017}\natexlab{}.
\newblock \showarticletitle{Constrained Policy Optimization}. In \bibinfo{booktitle}{\emph{Proceedings of the 34th International Conference on Machine Learning}} \emph{(\bibinfo{series}{Proceedings of Machine Learning Research}, Vol.~\bibinfo{volume}{70})}, \bibfield{editor}{\bibinfo{person}{Doina Precup} {and} \bibinfo{person}{Yee~Whye Teh}} (Eds.). \bibinfo{publisher}{{PMLR}}, \bibinfo{pages}{22--31}.
\newblock


\bibitem[\protect\citeauthoryear{Alegre, Roijers, Now{\'e}, Bazzan, and {da Silva}}{Alegre et~al\mbox{.}}{2023}]%
        {alegre2023sampleefficient}
\bibfield{author}{\bibinfo{person}{Lucas~N. Alegre}, \bibinfo{person}{Diederik~M. Roijers}, \bibinfo{person}{Ann Now{\'e}}, \bibinfo{person}{Ana L.~C. Bazzan}, {and} \bibinfo{person}{Bruno~C. {da Silva}}.} \bibinfo{year}{2023}\natexlab{}.
\newblock \showarticletitle{Sample-Efficient Multi-Objective Learning via Generalized Policy Improvement Prioritization}. In \bibinfo{booktitle}{\emph{Proc. of the 22nd International Conference on Autonomous Agents and Multiagent Systems ({{AAMAS}})}}.
\newblock


\bibitem[\protect\citeauthoryear{Altman}{Altman}{1999}]%
        {altman1999constrained}
\bibfield{author}{\bibinfo{person}{Eitan Altman}.} \bibinfo{year}{1999}\natexlab{}.
\newblock \bibinfo{booktitle}{\emph{Constrained {{Markov Decision Processes}}} (\bibinfo{edition}{1} ed.)}.
\newblock \bibinfo{publisher}{{Routledge}}, \bibinfo{address}{{Boca Raton}}.
\newblock
\showISBNx{978-1-315-14022-3}
\urldef\tempurl%
\url{https://doi.org/10.1201/9781315140223}
\showDOI{\tempurl}


\bibitem[\protect\citeauthoryear{Castelletti, Pianosi, and Restelli}{Castelletti et~al\mbox{.}}{2013}]%
        {castelletti2013multiobjective}
\bibfield{author}{\bibinfo{person}{A. Castelletti}, \bibinfo{person}{F. Pianosi}, {and} \bibinfo{person}{M. Restelli}.} \bibinfo{year}{2013}\natexlab{}.
\newblock \showarticletitle{A Multiobjective Reinforcement Learning Approach to Water Resources Systems Operation: {{Pareto}} Frontier Approximation in a Single Run}.
\newblock \bibinfo{journal}{\emph{Water Resources Research}} \bibinfo{volume}{49}, \bibinfo{number}{6} (\bibinfo{year}{2013}), \bibinfo{pages}{3476--3486}.
\newblock
\urldef\tempurl%
\url{https://doi.org/10.1002/wrcr.20295}
\showDOI{\tempurl}
\showeprint{https://agupubs.onlinelibrary.wiley.com/doi/pdf/10.1002/wrcr.20295}


\bibitem[\protect\citeauthoryear{Chatterjee, Majumdar, and Henzinger}{Chatterjee et~al\mbox{.}}{2006}]%
        {chatterjee2006markov}
\bibfield{author}{\bibinfo{person}{Krishnendu Chatterjee}, \bibinfo{person}{Rupak Majumdar}, {and} \bibinfo{person}{Thomas~A. Henzinger}.} \bibinfo{year}{2006}\natexlab{}.
\newblock \showarticletitle{Markov Decision Processes with Multiple Objectives}. In \bibinfo{booktitle}{\emph{{{STACS}} 2006}}, \bibfield{editor}{\bibinfo{person}{Bruno Durand} {and} \bibinfo{person}{Wolfgang Thomas}} (Eds.). \bibinfo{publisher}{{Springer Berlin Heidelberg}}, \bibinfo{address}{{Berlin, Heidelberg}}, \bibinfo{pages}{325--336}.
\newblock
\showISBNx{978-3-540-32288-7}


\bibitem[\protect\citeauthoryear{Delgrange, Katoen, Quatmann, and Randour}{Delgrange et~al\mbox{.}}{2020}]%
        {delgrange2020simple}
\bibfield{author}{\bibinfo{person}{Florent Delgrange}, \bibinfo{person}{Joost-Pieter Katoen}, \bibinfo{person}{Tim Quatmann}, {and} \bibinfo{person}{Mickael Randour}.} \bibinfo{year}{2020}\natexlab{}.
\newblock \showarticletitle{Simple Strategies in Multi-Objective {{MDPs}}}. In \bibinfo{booktitle}{\emph{Tools and Algorithms for the Construction and Analysis of Systems}}, \bibfield{editor}{\bibinfo{person}{Armin Biere} {and} \bibinfo{person}{David Parker}} (Eds.). \bibinfo{publisher}{{Springer International Publishing}}, \bibinfo{address}{{Cham}}, \bibinfo{pages}{346--364}.
\newblock
\showISBNx{978-3-030-45190-5}


\bibitem[\protect\citeauthoryear{Ding, Wei, Yang, Wang, and Jovanovic}{Ding et~al\mbox{.}}{2021}]%
        {ding2021provably}
\bibfield{author}{\bibinfo{person}{Dongsheng Ding}, \bibinfo{person}{Xiaohan Wei}, \bibinfo{person}{Zhuoran Yang}, \bibinfo{person}{Zhaoran Wang}, {and} \bibinfo{person}{Mihailo Jovanovic}.} \bibinfo{year}{2021}\natexlab{}.
\newblock \showarticletitle{Provably Efficient Safe Exploration via Primal-Dual Policy Optimization}. In \bibinfo{booktitle}{\emph{Proceedings of the 24th International Conference on Artificial Intelligence and Statistics}} \emph{(\bibinfo{series}{Proceedings of Machine Learning Research}, Vol.~\bibinfo{volume}{130})}, \bibfield{editor}{\bibinfo{person}{Arindam Banerjee} {and} \bibinfo{person}{Kenji Fukumizu}} (Eds.). \bibinfo{publisher}{{PMLR}}, \bibinfo{pages}{3304--3312}.
\newblock


\bibitem[\protect\citeauthoryear{Dolgov and Durfee}{Dolgov and Durfee}{2005}]%
        {dolgov2005stationary}
\bibfield{author}{\bibinfo{person}{Dmitri Dolgov} {and} \bibinfo{person}{Edmund Durfee}.} \bibinfo{year}{2005}\natexlab{}.
\newblock \showarticletitle{Stationary Deterministic Policies for Constrained {{MDPs}} with Multiple Rewards, Costs, and Discount Factors}. In \bibinfo{booktitle}{\emph{Proceedings of the 19th International Joint Conference on Artificial Intelligence}} \emph{(\bibinfo{series}{{{IJCAI}}'05})}. \bibinfo{publisher}{{Morgan Kaufmann Publishers Inc.}}, \bibinfo{address}{{San Francisco, CA, USA}}, \bibinfo{pages}{1326--1331}.
\newblock


\bibitem[\protect\citeauthoryear{Feinberg}{Feinberg}{2000}]%
        {feinberg2000constrained}
\bibfield{author}{\bibinfo{person}{Eugene~A. Feinberg}.} \bibinfo{year}{2000}\natexlab{}.
\newblock \showarticletitle{Constrained Discounted Markov Decision Processes and Hamiltonian Cycles}.
\newblock \bibinfo{journal}{\emph{Mathematics of Operations Research}} \bibinfo{volume}{25}, \bibinfo{number}{1} (\bibinfo{year}{2000}), \bibinfo{pages}{130--140}.
\newblock
\showISSN{0364765X, 15265471}
\showeprint[jstor]{3690427}


\bibitem[\protect\citeauthoryear{Felten, Alegre, Nowe, Bazzan, Talbi, Danoy, and da~Silva}{Felten et~al\mbox{.}}{2023}]%
        {felten2023toolkit}
\bibfield{author}{\bibinfo{person}{Florian Felten}, \bibinfo{person}{Lucas~Nunes Alegre}, \bibinfo{person}{Ann Nowe}, \bibinfo{person}{Ana L.~C. Bazzan}, \bibinfo{person}{El~Ghazali Talbi}, \bibinfo{person}{Gr{\'e}goire Danoy}, {and} \bibinfo{person}{Bruno~Castro da Silva}.} \bibinfo{year}{2023}\natexlab{}.
\newblock \showarticletitle{A {{Toolkit}} for {{Reliable Benchmarking}} and {{Research}} in {{Multi-Objective Reinforcement Learning}}}. In \bibinfo{booktitle}{\emph{Proceedings of the 37th {{Conference}} on {{Neural Information Processing Systems}} ({{NeurIPS}} 2023)}}.
\newblock


\bibitem[\protect\citeauthoryear{Felten, Talbi, and Danoy}{Felten et~al\mbox{.}}{2024}]%
        {felten2024multiobjective}
\bibfield{author}{\bibinfo{person}{Florian Felten}, \bibinfo{person}{El-Ghazali Talbi}, {and} \bibinfo{person}{Gr{\'e}goire Danoy}.} \bibinfo{year}{2024}\natexlab{}.
\newblock \showarticletitle{Multi-Objective Reinforcement Learning Based on Decomposition: {{A}} Taxonomy and Framework}.
\newblock \bibinfo{journal}{\emph{Journal of Artificial Intelligence Research}}  \bibinfo{volume}{79} (\bibinfo{year}{2024}), \bibinfo{pages}{679--723}.
\newblock
\urldef\tempurl%
\url{https://doi.org/10.1613/JAIR.1.15702}
\showDOI{\tempurl}


\bibitem[\protect\citeauthoryear{Geist, P{\'e}rolat, Lauri{\`e}re, Elie, Perrin, Bachem, Munos, and Pietquin}{Geist et~al\mbox{.}}{2022}]%
        {geist2022concave}
\bibfield{author}{\bibinfo{person}{Matthieu Geist}, \bibinfo{person}{Julien P{\'e}rolat}, \bibinfo{person}{Mathieu Lauri{\`e}re}, \bibinfo{person}{Romuald Elie}, \bibinfo{person}{Sarah Perrin}, \bibinfo{person}{Oliver Bachem}, \bibinfo{person}{R{\'e}mi Munos}, {and} \bibinfo{person}{Olivier Pietquin}.} \bibinfo{year}{2022}\natexlab{}.
\newblock \showarticletitle{Concave Utility Reinforcement Learning: {{The}} Mean-Field Game Viewpoint}. In \bibinfo{booktitle}{\emph{Proceedings of the 21st International Conference on Autonomous Agents and Multiagent Systems}} \emph{(\bibinfo{series}{{{AAMAS}} '22})}. \bibinfo{publisher}{{International Foundation for Autonomous Agents and Multiagent Systems}}, \bibinfo{address}{{Richland, SC}}, \bibinfo{pages}{489--497}.
\newblock
\showISBNx{978-1-4503-9213-6}


\bibitem[\protect\citeauthoryear{Hayes, R{\u a}dulescu, Bargiacchi, K{\"a}llstr{\"o}m, Macfarlane, Reymond, Verstraeten, Zintgraf, Dazeley, Heintz, Howley, Irissappane, Mannion, Now{\'e}, Ramos, Restelli, Vamplew, and Roijers}{Hayes et~al\mbox{.}}{2022}]%
        {hayes2022practical}
\bibfield{author}{\bibinfo{person}{Conor~F. Hayes}, \bibinfo{person}{Roxana R{\u a}dulescu}, \bibinfo{person}{Eugenio Bargiacchi}, \bibinfo{person}{Johan K{\"a}llstr{\"o}m}, \bibinfo{person}{Matthew Macfarlane}, \bibinfo{person}{Mathieu Reymond}, \bibinfo{person}{Timothy Verstraeten}, \bibinfo{person}{Luisa~M. Zintgraf}, \bibinfo{person}{Richard Dazeley}, \bibinfo{person}{Fredrik Heintz}, \bibinfo{person}{Enda Howley}, \bibinfo{person}{Athirai~A. Irissappane}, \bibinfo{person}{Patrick Mannion}, \bibinfo{person}{Ann Now{\'e}}, \bibinfo{person}{Gabriel Ramos}, \bibinfo{person}{Marcello Restelli}, \bibinfo{person}{Peter Vamplew}, {and} \bibinfo{person}{Diederik~M. Roijers}.} \bibinfo{year}{2022}\natexlab{}.
\newblock \showarticletitle{A Practical Guide to Multi-Objective Reinforcement Learning and Planning}.
\newblock \bibinfo{journal}{\emph{Autonomous Agents and Multi-Agent Systems}} \bibinfo{volume}{36}, \bibinfo{number}{1} (\bibinfo{date}{April} \bibinfo{year}{2022}), \bibinfo{pages}{26}.
\newblock
\showISSN{1573-7454}
\urldef\tempurl%
\url{https://doi.org/10.1007/s10458-022-09552-y}
\showDOI{\tempurl}


\bibitem[\protect\citeauthoryear{Legriel, Le~Guernic, Cotton, and Maler}{Legriel et~al\mbox{.}}{2010}]%
        {legriel2010approximating}
\bibfield{author}{\bibinfo{person}{Julien Legriel}, \bibinfo{person}{Colas Le~Guernic}, \bibinfo{person}{Scott Cotton}, {and} \bibinfo{person}{Oded Maler}.} \bibinfo{year}{2010}\natexlab{}.
\newblock \showarticletitle{Approximating the Pareto Front of Multi-Criteria Optimization Problems}. In \bibinfo{booktitle}{\emph{Tools and Algorithms for the Construction and Analysis of Systems}}, \bibfield{editor}{\bibinfo{person}{Javier Esparza} {and} \bibinfo{person}{Rupak Majumdar}} (Eds.). \bibinfo{publisher}{{Springer Berlin Heidelberg}}, \bibinfo{address}{{Berlin, Heidelberg}}, \bibinfo{pages}{69--83}.
\newblock
\showISBNx{978-3-642-12002-2}


\bibitem[\protect\citeauthoryear{Lu, Herman, and Yu}{Lu et~al\mbox{.}}{2023}]%
        {lu2023multiobjective}
\bibfield{author}{\bibinfo{person}{Haoye Lu}, \bibinfo{person}{Daniel Herman}, {and} \bibinfo{person}{Yaoliang Yu}.} \bibinfo{year}{2023}\natexlab{}.
\newblock \showarticletitle{Multi-Objective Reinforcement Learning: {{Convexity}}, Stationarity and Pareto Optimality}. In \bibinfo{booktitle}{\emph{The Eleventh International Conference on Learning Representations}}.
\newblock


\bibitem[\protect\citeauthoryear{Meng, Zheng, Pan, and Yin}{Meng et~al\mbox{.}}{2023}]%
        {meng2023offpolicy}
\bibfield{author}{\bibinfo{person}{Wenjia Meng}, \bibinfo{person}{Qian Zheng}, \bibinfo{person}{Gang Pan}, {and} \bibinfo{person}{Yilong Yin}.} \bibinfo{year}{2023}\natexlab{}.
\newblock \showarticletitle{Off-Policy Proximal Policy Optimization}.
\newblock \bibinfo{journal}{\emph{Proceedings of the AAAI Conference on Artificial Intelligence}} \bibinfo{volume}{37}, \bibinfo{number}{8} (\bibinfo{date}{June} \bibinfo{year}{2023}), \bibinfo{pages}{9162--9170}.
\newblock
\urldef\tempurl%
\url{https://doi.org/10.1609/aaai.v37i8.26099}
\showDOI{\tempurl}


\bibitem[\protect\citeauthoryear{Miettinen}{Miettinen}{1998}]%
        {miettinen1998nonlinear}
\bibfield{author}{\bibinfo{person}{Kaisa Miettinen}.} \bibinfo{year}{1998}\natexlab{}.
\newblock \bibinfo{booktitle}{\emph{Nonlinear {{Multiobjective Optimization}}}}. \bibinfo{series}{International {{Series}} in {{Operations Research}} \& {{Management Science}}}, Vol.~\bibinfo{volume}{12}.
\newblock \bibinfo{publisher}{{Springer US}}, \bibinfo{address}{{Boston, MA}}.
\newblock
\showISBNx{978-1-4613-7544-9 978-1-4615-5563-6}
\urldef\tempurl%
\url{https://doi.org/10.1007/978-1-4615-5563-6}
\showDOI{\tempurl}


\bibitem[\protect\citeauthoryear{Mnih, Badia, Mirza, Graves, Harley, Lillicrap, Silver, and Kavukcuoglu}{Mnih et~al\mbox{.}}{2016}]%
        {mnih2016asynchronous}
\bibfield{author}{\bibinfo{person}{Volodymyr Mnih}, \bibinfo{person}{Adria~Puigdomenech Badia}, \bibinfo{person}{Lehdi Mirza}, \bibinfo{person}{Alex Graves}, \bibinfo{person}{Tim Harley}, \bibinfo{person}{Timothy~P. Lillicrap}, \bibinfo{person}{David Silver}, {and} \bibinfo{person}{Koray Kavukcuoglu}.} \bibinfo{year}{2016}\natexlab{}.
\newblock \showarticletitle{Asynchronous Methods for Deep Reinforcement Learning}. In \bibinfo{booktitle}{\emph{33rd {{International Conference}} on {{Machine Learning}}, {{ICML}} 2016}}, \bibfield{editor}{\bibinfo{person}{Maria~Florina Balcan} {and} \bibinfo{person}{Kilian~Q Weinberger}} (Eds.), Vol.~\bibinfo{volume}{4}. \bibinfo{publisher}{{PMLR}}, \bibinfo{address}{{New York, New York, USA}}, \bibinfo{pages}{2850--2869}.
\newblock
\showISBNx{978-1-5108-2900-8}


\bibitem[\protect\citeauthoryear{Mnih, Kavukcuoglu, Silver, Rusu, Veness, Bellemare, Graves, Riedmiller, Fidjeland, Ostrovski, Petersen, Beattie, Sadik, Antonoglou, King, Kumaran, Wierstra, Legg, and Hassabis}{Mnih et~al\mbox{.}}{2015}]%
        {mnih2015humanlevel}
\bibfield{author}{\bibinfo{person}{Volodymyr Mnih}, \bibinfo{person}{Koray Kavukcuoglu}, \bibinfo{person}{David Silver}, \bibinfo{person}{Andrei~A Rusu}, \bibinfo{person}{Joel Veness}, \bibinfo{person}{Marc~G Bellemare}, \bibinfo{person}{Alex Graves}, \bibinfo{person}{Martin Riedmiller}, \bibinfo{person}{Andreas~K Fidjeland}, \bibinfo{person}{Georg Ostrovski}, \bibinfo{person}{Stig Petersen}, \bibinfo{person}{Charles Beattie}, \bibinfo{person}{Amir Sadik}, \bibinfo{person}{Ioannis Antonoglou}, \bibinfo{person}{Helen King}, \bibinfo{person}{Dharshan Kumaran}, \bibinfo{person}{Daan Wierstra}, \bibinfo{person}{Shane Legg}, {and} \bibinfo{person}{Demis Hassabis}.} \bibinfo{year}{2015}\natexlab{}.
\newblock \showarticletitle{Human-Level Control through Deep Reinforcement Learning}.
\newblock \bibinfo{journal}{\emph{Nature}} \bibinfo{volume}{518}, \bibinfo{number}{7540} (\bibinfo{year}{2015}), \bibinfo{pages}{529--533}.
\newblock
\urldef\tempurl%
\url{https://doi.org/10.1038/nature14236}
\showDOI{\tempurl}


\bibitem[\protect\citeauthoryear{Montenegro, Mussi, Metelli, and Papini}{Montenegro et~al\mbox{.}}{7 27}]%
        {montenegro2024learning}
\bibfield{author}{\bibinfo{person}{Alessandro Montenegro}, \bibinfo{person}{Marco Mussi}, \bibinfo{person}{Alberto~Maria Metelli}, {and} \bibinfo{person}{Matteo Papini}.} \bibinfo{year}{2024-07-21/2024-07-27}\natexlab{}.
\newblock \showarticletitle{Learning Optimal Deterministic Policies with Stochastic Policy Gradients}. In \bibinfo{booktitle}{\emph{Proceedings of the 41st International Conference on Machine Learning}} \emph{(\bibinfo{series}{Proceedings of Machine Learning Research}, Vol.~\bibinfo{volume}{235})}, \bibfield{editor}{\bibinfo{person}{Ruslan Salakhutdinov}, \bibinfo{person}{Zico Kolter}, \bibinfo{person}{Katherine Heller}, \bibinfo{person}{Adrian Weller}, \bibinfo{person}{Nuria Oliver}, \bibinfo{person}{Jonathan Scarlett}, {and} \bibinfo{person}{Felix Berkenkamp}} (Eds.). \bibinfo{publisher}{PMLR}, \bibinfo{pages}{36160--36211}.
\newblock


\bibitem[\protect\citeauthoryear{Nikulin, Miettinen, and M{\"a}kel{\"a}}{Nikulin et~al\mbox{.}}{2012}]%
        {nikulin2012new}
\bibfield{author}{\bibinfo{person}{Yury Nikulin}, \bibinfo{person}{Kaisa Miettinen}, {and} \bibinfo{person}{Marko~M. M{\"a}kel{\"a}}.} \bibinfo{year}{2012}\natexlab{}.
\newblock \showarticletitle{A New Achievement Scalarizing Function Based on Parameterization in Multiobjective Optimization}.
\newblock \bibinfo{journal}{\emph{OR Spectrum}} \bibinfo{volume}{34}, \bibinfo{number}{1} (\bibinfo{date}{Jan.} \bibinfo{year}{2012}), \bibinfo{pages}{69--87}.
\newblock
\showISSN{1436-6304}
\urldef\tempurl%
\url{https://doi.org/10.1007/s00291-010-0224-1}
\showDOI{\tempurl}


\bibitem[\protect\citeauthoryear{Papadimitriou and Yannakakis}{Papadimitriou and Yannakakis}{2000}]%
        {papadimitriou2000approximability}
\bibfield{author}{\bibinfo{person}{C.H. Papadimitriou} {and} \bibinfo{person}{M. Yannakakis}.} \bibinfo{year}{2000}\natexlab{}.
\newblock \showarticletitle{On the Approximability of Trade-Offs and Optimal Access of {{Web}} Sources}. In \bibinfo{booktitle}{\emph{Proceedings 41st Annual Symposium on Foundations of Computer Science}}. \bibinfo{pages}{86--92}.
\newblock
\showISSN{0272-5428}
\urldef\tempurl%
\url{https://doi.org/10.1109/SFCS.2000.892068}
\showDOI{\tempurl}


\bibitem[\protect\citeauthoryear{Reymond, Bargiacchi, and Now{\'e}}{Reymond et~al\mbox{.}}{2022}]%
        {reymond2022pareto}
\bibfield{author}{\bibinfo{person}{Mathieu Reymond}, \bibinfo{person}{Eugenio Bargiacchi}, {and} \bibinfo{person}{Ann Now{\'e}}.} \bibinfo{year}{2022}\natexlab{}.
\newblock \showarticletitle{Pareto Conditioned Networks}. In \bibinfo{booktitle}{\emph{Proceedings of the 21st International Conference on Autonomous Agents and Multiagent Systems}} \emph{(\bibinfo{series}{{{AAMAS}} '22})}. \bibinfo{publisher}{{International Foundation for Autonomous Agents and Multiagent Systems}}, \bibinfo{address}{{Richland, SC}}, \bibinfo{pages}{1110--1118}.
\newblock
\showISBNx{978-1-4503-9213-6}


\bibitem[\protect\citeauthoryear{Reymond, Hayes, Steckelmacher, Roijers, and Now{\'e}}{Reymond et~al\mbox{.}}{2023}]%
        {reymond2023actorcritic}
\bibfield{author}{\bibinfo{person}{Mathieu Reymond}, \bibinfo{person}{Conor~F. Hayes}, \bibinfo{person}{Denis Steckelmacher}, \bibinfo{person}{Diederik~M. Roijers}, {and} \bibinfo{person}{Ann Now{\'e}}.} \bibinfo{year}{2023}\natexlab{}.
\newblock \showarticletitle{Actor-Critic Multi-Objective Reinforcement Learning for Non-Linear Utility Functions}.
\newblock \bibinfo{journal}{\emph{Autonomous Agents and Multi-Agent Systems}} \bibinfo{volume}{37}, \bibinfo{number}{2} (\bibinfo{date}{April} \bibinfo{year}{2023}), \bibinfo{pages}{23}.
\newblock
\showISSN{1573-7454}
\urldef\tempurl%
\url{https://doi.org/10.1007/s10458-023-09604-x}
\showDOI{\tempurl}


\bibitem[\protect\citeauthoryear{Roijers and Whiteson}{Roijers and Whiteson}{2017}]%
        {roijers2017multiobjective}
\bibfield{author}{\bibinfo{person}{Diederik~M. Roijers} {and} \bibinfo{person}{Shimon Whiteson}.} \bibinfo{year}{2017}\natexlab{}.
\newblock \showarticletitle{Multi-Objective Decision Making}. In \bibinfo{booktitle}{\emph{Synthesis {{Lectures}} on {{Artificial Intelligence}} and {{Machine Learning}}}}, Vol.~\bibinfo{volume}{34}. \bibinfo{publisher}{{Morgan and Claypool}}, \bibinfo{pages}{129--129}.
\newblock
\showISBNx{978-1-62705-960-2}
\urldef\tempurl%
\url{https://doi.org/10.2200/S00765ED1V01Y201704AIM034}
\showDOI{\tempurl}


\bibitem[\protect\citeauthoryear{Schulman, Wolski, Dhariwal, Radford, and Klimov}{Schulman et~al\mbox{.}}{2017}]%
        {schulman2017proximal}
\bibfield{author}{\bibinfo{person}{John Schulman}, \bibinfo{person}{Filip Wolski}, \bibinfo{person}{Prafulla Dhariwal}, \bibinfo{person}{Alec Radford}, {and} \bibinfo{person}{Oleg Klimov}.} \bibinfo{year}{2017}\natexlab{}.
\newblock \showarticletitle{Proximal Policy Optimization Algorithms}.
\newblock \bibinfo{journal}{\emph{arXiv preprint arXiv:1707.06347}} (\bibinfo{year}{2017}).
\newblock
\showeprint[arxiv]{1707.06347}


\bibitem[\protect\citeauthoryear{Siddique, Weng, and Zimmer}{Siddique et~al\mbox{.}}{2020}]%
        {siddique2020learning}
\bibfield{author}{\bibinfo{person}{Umer Siddique}, \bibinfo{person}{Paul Weng}, {and} \bibinfo{person}{Matthieu Zimmer}.} \bibinfo{year}{2020}\natexlab{}.
\newblock \showarticletitle{Learning Fair Policies in Multi-Objective ({{Deep}}) Reinforcement Learning with Average and Discounted Rewards}. In \bibinfo{booktitle}{\emph{Proceedings of the 37th International Conference on Machine Learning}} \emph{(\bibinfo{series}{Proceedings of Machine Learning Research}, Vol.~\bibinfo{volume}{119})}, \bibfield{editor}{\bibinfo{person}{Hal~Daum{\'e} III} {and} \bibinfo{person}{Aarti Singh}} (Eds.). \bibinfo{publisher}{{PMLR}}, \bibinfo{pages}{8905--8915}.
\newblock


\bibitem[\protect\citeauthoryear{Vamplew, Dazeley, Berry, Issabekov, and Dekker}{Vamplew et~al\mbox{.}}{2011}]%
        {vamplew2011empirical}
\bibfield{author}{\bibinfo{person}{Peter Vamplew}, \bibinfo{person}{Richard Dazeley}, \bibinfo{person}{Adam Berry}, \bibinfo{person}{Rustam Issabekov}, {and} \bibinfo{person}{Evan Dekker}.} \bibinfo{year}{2011}\natexlab{}.
\newblock \showarticletitle{Empirical Evaluation Methods for Multiobjective Reinforcement Learning Algorithms}.
\newblock \bibinfo{journal}{\emph{Machine Learning}} \bibinfo{volume}{84}, \bibinfo{number}{1} (\bibinfo{year}{2011}), \bibinfo{pages}{51--80}.
\newblock
\urldef\tempurl%
\url{https://doi.org/10.1007/s10994-010-5232-5}
\showDOI{\tempurl}


\bibitem[\protect\citeauthoryear{Van~Moffaert, Drugan, and Now{\'e}}{Van~Moffaert et~al\mbox{.}}{2013}]%
        {vanmoffaert2013scalarized}
\bibfield{author}{\bibinfo{person}{Kristof Van~Moffaert}, \bibinfo{person}{Madalina~M. Drugan}, {and} \bibinfo{person}{Ann Now{\'e}}.} \bibinfo{year}{2013}\natexlab{}.
\newblock \showarticletitle{Scalarized Multi-Objective Reinforcement Learning: {{Novel}} Design Techniques}. In \bibinfo{booktitle}{\emph{2013 {{IEEE}} Symposium on Adaptive Dynamic Programming and Reinforcement Learning ({{ADPRL}})}}. \bibinfo{pages}{191--199}.
\newblock
\urldef\tempurl%
\url{https://doi.org/10.1109/ADPRL.2013.6615007}
\showDOI{\tempurl}


\bibitem[\protect\citeauthoryear{White}{White}{1982}]%
        {white1982multiobjective}
\bibfield{author}{\bibinfo{person}{D~J White}.} \bibinfo{year}{1982}\natexlab{}.
\newblock \showarticletitle{Multi-Objective Infinite-Horizon Discounted {{Markov}} Decision Processes}.
\newblock \bibinfo{journal}{\emph{J. Math. Anal. Appl.}} \bibinfo{volume}{89}, \bibinfo{number}{2} (\bibinfo{year}{1982}), \bibinfo{pages}{639--647}.
\newblock
\urldef\tempurl%
\url{https://doi.org/10.1016/0022-247X(82)90122-6}
\showDOI{\tempurl}


\bibitem[\protect\citeauthoryear{Wierzbicki}{Wierzbicki}{1982}]%
        {wierzbicki1982mathematical}
\bibfield{author}{\bibinfo{person}{Andrzej~P. Wierzbicki}.} \bibinfo{year}{1982}\natexlab{}.
\newblock \showarticletitle{A Mathematical Basis for Satisficing Decision Making}.
\newblock \bibinfo{journal}{\emph{Mathematical Modelling}} \bibinfo{volume}{3}, \bibinfo{number}{5} (\bibinfo{year}{1982}), \bibinfo{pages}{391--405}.
\newblock
\showISSN{0270-0255}
\urldef\tempurl%
\url{https://doi.org/10.1016/0270-0255(82)90038-0}
\showDOI{\tempurl}


\bibitem[\protect\citeauthoryear{Xu, Tian, Ma, Rus, Sueda, and Matusik}{Xu et~al\mbox{.}}{2020}]%
        {xu2020predictionguided}
\bibfield{author}{\bibinfo{person}{Jie Xu}, \bibinfo{person}{Yunsheng Tian}, \bibinfo{person}{Pingchuan Ma}, \bibinfo{person}{Daniela Rus}, \bibinfo{person}{Shinjiro Sueda}, {and} \bibinfo{person}{Wojciech Matusik}.} \bibinfo{year}{2020}\natexlab{}.
\newblock \showarticletitle{Prediction-{{Guided Multi-Objective Reinforcement Learning}} for {{Continuous Robot Control}}}. In \bibinfo{booktitle}{\emph{Proceedings of the 37th {{International Conference}} on {{Machine Learning}}}}, \bibfield{editor}{\bibinfo{person}{Hal~Daum{\'e} III} {and} \bibinfo{person}{Aarti Singh}} (Eds.), Vol.~\bibinfo{volume}{119}. \bibinfo{publisher}{{PMLR}}, \bibinfo{pages}{10607--10616}.
\newblock


\bibitem[\protect\citeauthoryear{Yang, Sun, and Narasimhan}{Yang et~al\mbox{.}}{2019}]%
        {yang2019generalized}
\bibfield{author}{\bibinfo{person}{Runzhe Yang}, \bibinfo{person}{Xingyuan Sun}, {and} \bibinfo{person}{Karthik Narasimhan}.} \bibinfo{year}{2019}\natexlab{}.
\newblock \showarticletitle{A Generalized Algorithm for Multi-Objective Reinforcement Learning and Policy Adaptation}.
\newblock In \bibinfo{booktitle}{\emph{Proceedings of the 33rd International Conference on Neural Information Processing Systems}}. \bibinfo{publisher}{{Curran Associates Inc.}}, \bibinfo{address}{{Red Hook, NY, USA}}.
\newblock


\bibitem[\protect\citeauthoryear{Zahavy, O'Donoghue, Desjardins, and Singh}{Zahavy et~al\mbox{.}}{2021}]%
        {zahavy2021reward}
\bibfield{author}{\bibinfo{person}{Tom Zahavy}, \bibinfo{person}{Brendan O'Donoghue}, \bibinfo{person}{Guillaume Desjardins}, {and} \bibinfo{person}{Satinder Singh}.} \bibinfo{year}{2021}\natexlab{}.
\newblock \showarticletitle{Reward Is Enough for Convex {{MDPs}}}. In \bibinfo{booktitle}{\emph{Advances in Neural Information Processing Systems}}, \bibfield{editor}{\bibinfo{person}{A.~Beygelzimer}, \bibinfo{person}{Y.~Dauphin}, \bibinfo{person}{P.~Liang}, {and} \bibinfo{person}{J.~Wortman Vaughan}} (Eds.).
\newblock


\bibitem[\protect\citeauthoryear{Zhang, Koppel, Bedi, Szepesvari, and Wang}{Zhang et~al\mbox{.}}{2020}]%
        {zhang2020variational}
\bibfield{author}{\bibinfo{person}{Junyu Zhang}, \bibinfo{person}{Alec Koppel}, \bibinfo{person}{Amrit~Singh Bedi}, \bibinfo{person}{Csaba Szepesvari}, {and} \bibinfo{person}{Mengdi Wang}.} \bibinfo{year}{2020}\natexlab{}.
\newblock \showarticletitle{Variational Policy Gradient Method for Reinforcement Learning with General Utilities}. In \bibinfo{booktitle}{\emph{Advances in Neural Information Processing Systems}}, \bibfield{editor}{\bibinfo{person}{H.~Larochelle}, \bibinfo{person}{M.~Ranzato}, \bibinfo{person}{R.~Hadsell}, \bibinfo{person}{M.F. Balcan}, {and} \bibinfo{person}{H.~Lin}} (Eds.), Vol.~\bibinfo{volume}{33}. \bibinfo{publisher}{{Curran Associates, Inc.}}, \bibinfo{pages}{4572--4583}.
\newblock


\bibitem[\protect\citeauthoryear{Zhang and Li}{Zhang and Li}{2007}]%
        {zhang2007moea}
\bibfield{author}{\bibinfo{person}{Qingfu Zhang} {and} \bibinfo{person}{Hui Li}.} \bibinfo{year}{2007}\natexlab{}.
\newblock \showarticletitle{{{MOEA}}/{{D}}: {{A}} Multiobjective Evolutionary Algorithm Based on Decomposition}.
\newblock \bibinfo{journal}{\emph{IEEE Transactions on Evolutionary Computation}} \bibinfo{volume}{11}, \bibinfo{number}{6} (\bibinfo{date}{Dec.} \bibinfo{year}{2007}), \bibinfo{pages}{712--731}.
\newblock
\showISSN{1941-0026}
\urldef\tempurl%
\url{https://doi.org/10.1109/TEVC.2007.892759}
\showDOI{\tempurl}


\bibitem[\protect\citeauthoryear{Zintgraf, Kanters, Roijers, Oliehoek, and Beau}{Zintgraf et~al\mbox{.}}{2015}]%
        {zintgraf2015quality}
\bibfield{author}{\bibinfo{person}{L~M Zintgraf}, \bibinfo{person}{T~V Kanters}, \bibinfo{person}{Diederik~M. Roijers}, \bibinfo{person}{F~A Oliehoek}, {and} \bibinfo{person}{P Beau}.} \bibinfo{year}{2015}\natexlab{}.
\newblock \showarticletitle{Quality {{Assessment}} of {{MORL Algorithms}}: {{A Utility-Based Approach}}}.
\newblock \bibinfo{journal}{\emph{Proc Belgian-Dutch Conf on Machine Learning}} (\bibinfo{year}{2015}).
\newblock


\end{thebibliography}


\newpage
\appendix
\onecolumn

\section{Theoretical results for IPRO}
\label{ap:ipro-proofs}
In this section, we provide the omitted proofs for IPRO from \cref{sec:ipro}. These results establish both the upper bound on the true approximation error and convergence guarantees to the true Pareto front in the limit or to an approximate Pareto front in a finite number of iterations.

\subsection{Definitions}
\label{ap:ipro-defs}
Before presenting the proofs for IPRO, it is necessary to define the sets that are tracked in IPRO. Let $\bbox = \prod_{j=1}^d[v^{\text{n}}_j, v^{\text{i}}_j]$ be the bounding box defined by a strict lower bound of the true nadir $\nadir$ and the ideal $\ideal$. The set $\vset_t$ contains the Pareto front obtained in timestep $t$, while the completed set $\cset_t$ contains the referents for which the evaluation of the Pareto oracle was unsuccessful. We then define the dominated set $\dset$ and the infeasible set $\iset$ as follows.

\begin{definition}
\label{def:dominated-set}
The dominated set $\dset_t$ at timestep $t$ contains all points in the bounding box that are dominated by or equal to a point in the current Pareto front, that is, $\dset_t = \left\{\vv \in \bbox \mid \exists \vv' \in \vset_t, \vv' \pde \vv \right\}$.
\end{definition}

\begin{definition}
\label{def:infeasible-set}
The infeasible set $\iset_t$ at timestep $t$ contains all points in the bounding box that dominate or are equal to a point in the union of the current Pareto front and completed referents, i.e. $\iset_t = \left\{\vv \in B \mid \exists \vv' \in \vset_t \cup C_t, \vv \pde \vv' \right\}$.
\end{definition}

Note that in the definition of the infeasible set, we consider not only those points dominated by the current Pareto front but also the points dominated by the referents that failed to result in new solutions. 

During the execution of IPRO, it is necessary to recognise the remaining unexplored sections. For this, we define the boundaries, interiors and \emph{reachable} boundaries of the dominated and infeasible set. Let $\overline{S}$ be the closure of a subset $S$ in some topological space and $\boundary S$ be its boundary. By a slight abuse of notation, we say that $\boundary \dset_t = \overline{(\bbox \setminus \dset_t)} \cap \overline{\dset_t}$ is the boundary of $\dset_t$ in $\bbox$ and $\interior \dset_t = \dset_t \setminus \boundary \dset_t$ is the interior of $\dset_t$ in $\bbox$. We define the boundary and interior of the infeasible set analogously. Finally, we define the \emph{reachable} boundaries of these sets, which together delineate the remaining available search space and illustrate all defined subsets in \cref{fig:reachable-boundaries} for the two-dimensional case.

\begin{definition}
\label{def:reachable-boundary-d}
The reachable boundary of $\dset_t$, denoted $\rboundary \dset_t$ at timestep $t$ is defined as $\rboundary \dset_t = \boundary \dset_t \setminus \iset_t$.
\end{definition}

\begin{definition}
\label{def:reachable-boundary-i}
The reachable boundary of $\iset_t$, denoted $\rboundary \iset_t$ at timestep $t$ is defined as $\rboundary \iset_t = \boundary \iset_t \setminus \dset_t$.
\end{definition}

\begin{figure}[htb]
    \centering
        \begin{subfigure}[b]{0.3\textwidth}
        \centering
        \includegraphics[width=\textwidth]{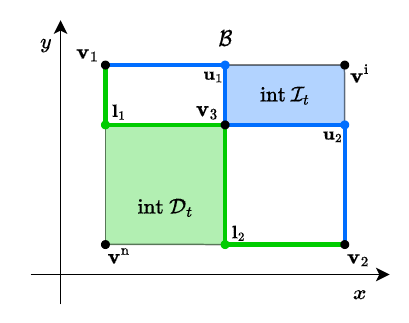}
        \subcaption{}
        \label{fig:ipro-rb-stage1}
    \end{subfigure}
    \qquad
    \begin{subfigure}[b]{0.3\textwidth}
        \centering
        \includegraphics[width=\textwidth]{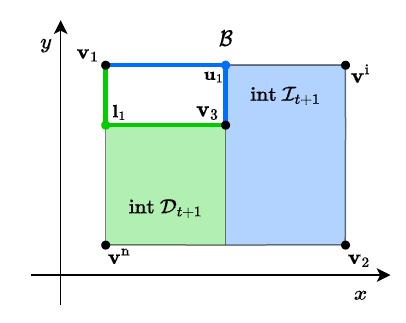}
        \subcaption{}
        \label{fig:ipro-rb-stage2}
    \end{subfigure}
    \caption{(a) The reachable boundaries of $\dset_t$ (green) and $\iset_t$ (blue) indicated with solid lines and their interiors (shaded) when no section is completed. (b) When completing the section at $\vl_2$, parts of the reachable boundary at timestep $t$ become unreachable at timestep $t+1$.}
    \label{fig:reachable-boundaries}
\end{figure}

For the reachable boundaries of the dominated and infeasible sets, we define two important subsets containing the respective lower and upper bounds of the remaining solutions on the Pareto front. The set of lower bounds $\lset$ contains the points on the reachable boundary of the dominated set such that no other point on the reachable boundary exists which is dominated by it. Similarly, the set of upper bounds $\uset$ contains the points on the reachable boundary of the infeasible set such that there is there is no other point on the reachable boundary that dominates it. Conceptually, these points are the inner corners of their respective boundary as observed in \cref{fig:reachable-boundaries}.

\begin{definition}[Lower Bounds]
\label{def:lower-set}
The set of lower bounds at timestep $t$ is defined as $\lset_t = \left\{\vl \in \rboundary \dset_t \mid \nexists \vv \in \rboundary \dset_t, \vl \pd \vv \right\}$.
\end{definition}
  
\begin{definition}[Upper Bounds]
\label{def:upper-set}
The set of upper bounds at timestep $t$ is defined as $\uset_t = \left\{\vu \in \rboundary \iset_t \mid \nexists \vv \in \rboundary \iset_t, \vv \pd \vu \right\}$.
\end{definition}

\subsection{Assumptions}
\label{sec:assumptions}
We explicitly state and motivate the assumptions underpinning our theoretical results. We emphasise that all assumptions are either on an implementation level, efficiently verifiable or guaranteed to hold for significant domains of interest.

First, we assume that the problem is not trivial and there exist unexplored regions in the bounding box $\bbox$ after finding the first $d$ weakly Pareto optimal solutions. This assumption is not hindering, since after the initialisation phase we can run a pruning algorithm such as $\textsc{PPrune}$ \citep{roijers2017multiobjective}, which takes as input a set of vectors and outputs only the Pareto optimal ones, and terminate IPRO if this is the case. 
\begin{assumption}
\label{assumption:degenerate}
For any MOMDP $\momdp$ with $d$-dimensional reward function, we assume that for the initial Pareto front $\vset_0 = \{\vv_1, \dotsc, \vv_d\}$ it is guaranteed that $|\textsc{PPrune}(\vset_0) | > 1$.
\end{assumption}

In addition, we provide assumptions necessary for weak Pareto oracles to ensure their convergence to the true Pareto front in the limit. Intuitively, we first assume that the referent selection mechanism is not antagonistic and we select in some iterations a referent that may reduce the error estimate. Note that this can be readily implemented using a randomised method which assigns a strictly positive probability to each lower bound in $\lset$. Alternatively, we can explicitly construct the set of lower bounds that are expected to reduce the error and select from this set rather than the entire set of lower bounds.

\begin{assumption}
\label{assumption:weak-oracle-referent-selection}
Let $\varepsilon_t = \max_{\vu \in \uset_t} \min_{\vv \in \vset_t} \|\vu - \vv\|_\infty$ be the upper bound on the true error $\varepsilon^\ast_t$ at timestep $t$. We define $\uset^\varepsilon_t \subseteq \uset_t$ to be the subset of upper bounds for which the error is equal to $\varepsilon_t$ and $\lset^\varepsilon_t \subseteq \lset_t$ to be the subset of lower bounds such that for all $\vl \in \lset^\varepsilon_t$ there is an $\vu \in \uset^\varepsilon_t: \vu > \vl$. As $t \to \infty$, an $\vl \in \lset^\varepsilon_t$ is almost surely selected as a referent.
\end{assumption}

The next assumption ensures that the oracle is not antagonistic and that it is capable of yielding any Pareto optimal solution. In other words, no solution is excluded by the Pareto oracle a priori. While this may be challenging to verify, in practice this can be satisfied by implementing a robust oracle.

\begin{assumption}
\label{assumption:weak-oracle-strength}
Let $\oracle$ be a weak Pareto oracle with tolerance $\tau=0$. For all undiscovered Pareto optimal solutions $\vast \in \truepf \setminus \vset_t$ there exists some lower bound $\vl$ such that $\oracle(\vl) = \vast$ and as $t \to \infty$, $\vl$ is almost surely added to $\lset_t$.
\end{assumption}

The final assumption that is necessary is on the shape of the Pareto front. Concretely, we assume that every segment of the Pareto front contains its endpoints. This is necessary to ensure that we may close all gaps between segments eventually, as otherwise we are never able to reduce the error estimate of IPRO below that of the largest gap. Importantly, \cref{assumption:weak-oracle-pf-shape} holds when considering stochastic policies in finite MOMDPs, since the set of occupancy measures is a closed convex polytope \citep{altman1999constrained}.

\begin{assumption}
\label{assumption:weak-oracle-pf-shape}
The Pareto front is the union of a finite number of paths (i.e. a continuous function $f: [0, 1] \to \mathbb{R}^d$).
\end{assumption}

\subsection{Supporting lemmas}
\label{ap:supporting-lemmas}
We provide supporting lemmas that formalise the contents of the sets defined in \cref{ap:ipro-defs} and their relation to the remaining feasible solutions. Concretely, we first demonstrate that the interior of the infeasible set contains only infeasible points or points within the acceptable tolerance, which is a consequence of having a strictly positive distance to the boundary. Combined with the dominated set, which inherently contains only dominated solutions, we can then significantly reduce the search space that is left to explore.

\begin{lemma}
\label{lemma:interior-infeasible}
Given an oracle $\oracle$ with tolerance $\tau$, then at all timesteps $t$ and for all $\vi \in \iset_t$, $\vi$ is infeasible or within the tolerance $\tau$ of a point on the current estimate of the Pareto front $\vset_t$.
\end{lemma}
\begin{proof}
Recall that the interior of the infeasible set is defined as follows,
\begin{equation}
    \interior \iset_t = \iset_t \setminus \boundary \iset_t.
\end{equation}
Let $\vv \in \interior \iset_t$ be a point in the interior of the infeasible set. Then there exists an open ball centred around $\vv$ with a strictly positive radius $r$ such that $B_r(\vv) \subseteq \interior \iset_t$. Let $\vv' \in B_r(\vv)$ be a point in the ball such that $\vv > \vv'$ which can be obtained by taking $\vv$ and subtracting a value $\delta \in (0, r)$. Since $\vv' \in \interior \iset_t$, the definition of the infeasible set (\cref{def:infeasible-set}) ensures that there exists a point $\bar{\vv} \in \vset_t \cup C_t$ such that $\vv' \pde \bar{\vv}$. By the transitivity of Pareto dominance, we then have that $\vv > \bar{\vv}$.

Let us now consider the two cases for $\bar{\vv}$. Assume first that $\bar{\vv} \in \vset_t$. If $\vv$ is a feasible solution and knowing that $\vv > \bar{\vv}$ implies that $\bar{\vv}$ is not weakly Pareto optimal. Therefore, $\bar{\vv}$ would not have been returned by a weak or approximate Pareto oracle. As such, $\vv$ must be infeasible.

When $\bar{\vv} \in C_t$ it was added after the oracle evaluation at $\bar\vv$ was unsuccessful. For a weak Pareto oracle $\oracle$ with tolerance $\tau = 0$, this again guarantees that $\vv$ is infeasible since $\vv > \bar\vv$. For an approximate Pareto oracle $\oracle$ with tolerance $\tau > 0$, we distinguish between two cases. If $\vv \pde \bar\vv + \tau$, $\vv$ is infeasible since $\oracle$ would otherwise have returned it. Finally, by the construction of the set of lower bounds $\lset$, there must exist a point $\vast$ on the current Pareto front $\vset_t$ such that $\vast + \tau \pde \vv$ and therefore $\vv$ was within the tolerance.
\end{proof}

Given the result for the infeasible solutions, we now focus instead on the remaining feasible solutions. Here, we demonstrate that all feasible solutions are strictly lower bounded by $\lset$ and upper bounded by $\uset$.

\begin{lemma}
\label{lemma:lower-strict}
At any timestep $t$, the set of lower bounds $\lset_t$ contains a strict lower bound for all remaining feasible solutions, i.e.,
\begin{equation}
\label{eq:lower-strict}
    \vv \in B \setminus (\interior \iset_t \cup \dset_t) \implies \exists \vl \in \lset_t, \vv >\vl.
\end{equation}
\end{lemma}
\begin{proof}
Let $\vv$ be a remaining feasible solution. Then it cannot be in the dominated set, as this implies it is dominated by a point on the current Pareto front, nor can it be in the interior of $\iset_t$ as this was guaranteed to be infeasible or within the tolerance following \cref{lemma:interior-infeasible}. However, $\vv$ can still be on the reachable boundary of the infeasible set when using a weak Pareto oracle. As such, we may indeed write in \cref{eq:lower-strict} that $\vv \in B \setminus (\interior \iset_t \cup \dset_t)$. 

Recall that in IPRO, the nadir $\nadir$ of the bounding box $\bbox$ is initialised to a guaranteed strict lower bound of the true nadir. Therefore, for all $\vv \in B \setminus (\interior \iset_t \cup \dset_t)$ we can connect a strictly decreasing line segment between $\vv$ and $\nadir$. Moreover, either $\nadir \in \boundary \dset_t$ or this line must intersect $\boundary \dset_t$ at some point $\bar{\vv}$ for which it is subsequently guaranteed that $\vv > \bar{\vv}$.

Let $\vv \in B \setminus (\interior \iset_t \cup \dset_t)$ be a feasible solution and $\bar{\vv} \in \boundary \dset_t$ be a point on the boundary of $\dset_t$ such that $\vv > \bar{\vv}$. Suppose, however, that $\bar{\vv}$ is not on the reachable boundary. Then, the definition of the reachable boundary implies that $\bar{\vv} \in \iset_t$ (see \cref{def:reachable-boundary-d}). However, as $\vv > \bar{\vv}$ this implies that $\vv$ is in the interior of $\iset_t$ which was guaranteed to be infeasible or within the tolerance by \cref{lemma:interior-infeasible}. Therefore, $\bar{\vv}$ must be on the reachable boundary of $\dset_t$. By definition of the lower set, this further implies there exists a lower point $\vl \in \lset_t$ for which $\bar{\vv} \pde \vl$, finally guaranteeing that $\vv > \vl$.
\end{proof}

We provide an analogous result for the upper set where we demonstrate that it contains an upper bound for all remaining feasible solutions.

\begin{lemma}
\label{lemma:upper}
During IPRO's execution, the upper set contains an upper bound for all remaining feasible solutions, i.e.,
\begin{equation}
\label{eq:upper}
    \vv \in B \setminus (\interior \iset_t \cup \dset_t) \implies \exists \vu \in \uset_t, \vu \pde \vv.
\end{equation}
\end{lemma}
\begin{proof}
As the ideal $\ideal$ is initialised to the true ideal, we may apply the same proof as for \cref{lemma:lower-strict} using Pareto dominance rather than strict Pareto dominance. In contrast to \cref{lemma:lower-strict} however, $\vv$ may be on the reachable boundary of the infeasible set $\rboundary \iset_t$. In this case, the definition of the set of upper bounds $\uset$ guarantees the existence of an upper bound $\vu \in \uset_t, \vu \pde \vv$.
\end{proof}

\subsection{Proof of Theorem 4.1}
\label{ap:proof-th41}
We now prove \cref{th:approximation-guarantee} which guarantees an upper bound on the true approximation error at any timestep. In fact, this upper bound follows almost immediately from the supporting lemmas shown in \cref{ap:supporting-lemmas}. Utilising the fact that the set of upper bounds contains a guaranteed upper bound for all remaining feasible solutions, we can compute the point that maximises the distance to its closest point on the current approximation of the Pareto front. Recall that at timestep $t$ the true approximation error $\varepsilon^\ast_t$ is defined as $\sup_{\vast \in \vset^\ast} \min_{\vv \in \vset_t} \|\vast - \vv\|_\infty$.

\begin{theorem41}
Let $\vset^\ast$ be the true Pareto front, $\vset_t$ the approximate Pareto front obtained by IPRO and $\varepsilon^\ast_t$ the true approximation error at timestep $t$. Then the following inequality holds,
\begin{equation}
    \varepsilon^\ast_t \leq \max_{\vu \in \uset_t} \min_{\vv \in \vset_t} \|\vu - \vv\|_\infty.
\end{equation}
\end{theorem41}
\begin{proof}
Observe that all remaining Pareto optimal solutions must be feasible and we can therefore derive from \cref{lemma:interior-infeasible,lemma:upper} that
\begin{equation}
\label{eq:remaining-vast-u}
    \forall t \in \mathbb{N}, \forall \vast \in \vset^\ast \setminus \vset_t, \exists \vu \in \uset_t: \vu \pde \vast.
\end{equation} 
From \cref{eq:remaining-vast-u} we can then conclude the following upper bound,
\begin{equation}
    \varepsilon^\ast_t = \sup_{\vast \in \vset^\ast \setminus \vset_t} \min_{\vv \in \vset_t} \|\vast - \vv\|_\infty \leq \max_{\vu \in \uset_t} \min_{\vv \in \vset_t} \|\vu - \vv\|_\infty.
\end{equation}
Note that this holds as the maximum over the upper points is guaranteed to be at least as high as the maximum over all remaining points on the Pareto front.
\end{proof}

A useful corollary of \cref{th:approximation-guarantee} is that the sequence of errors is monotonically decreasing. This follows immediately since the upper bounds are only adjusted downwards.

\begin{corollary}
\label{co:decreasing-error}
The sequence of errors $(\varepsilon_t)_{t \in \mathbb{N}}$ is monotonically decreasing.
\end{corollary}
\begin{proof}
Observe that since IPRO only adds points to the Pareto front, it is guaranteed that $\vset_{t} \subseteq \vset_{t+1}$. Furthermore, from the definition of the set of upper bounds, it is guaranteed that all remaining feasible solutions are upper bounded by a point in this set. Therefore, for all points in the updated set of upper bounds $\vu \in \uset_{t+1}$ there must exist an old upper bound $\bar{\vu} \in \uset_{t}$ such that $\bar{\vu} \pde \vu$. As such, we conclude that 
\begin{equation}
    \max_{\vu \in \uset_{t+1}} \min_{\vv \in \vset_{t+1}} \|\vu - \vv\|_\infty \leq \max_{\vu \in \uset_{t}} \min_{\vv \in \vset_{t}} \|\vu - \vv\|_\infty
\end{equation}
and thus $\forall t \in \mathbb{N}: \varepsilon_{t+1} \leq \varepsilon_{t}$.
\end{proof}

\subsection{Proof of Theorem 4.2}
\label{ap:proof-th42}
We show that IPRO is guaranteed to converge to a $\tau$-Pareto front when using an approximate Pareto oracle with tolerance $\tau > 0$. Moreover, when using a weak Pareto oracle, the $\tau$ may be set to $0$ and the true Pareto front is obtained in the limit. For practical purposes, however, setting $\tau > 0$ ensures that IPRO converges after a finite number of iterations.

\begin{theorem}
\label{th:approximate-po-convergence}
Given an approximate Pareto oracle $\oracle$ with tolerance $\tau > 0$, IPRO converges to a $\tau$-Pareto front in a finite number of iterations.
\end{theorem}
\begin{proof}
From \cref{co:decreasing-error}, we know that the sequence of errors produced by IPRO is monotonically decreasing. We show that this sequence converges to zero when ignoring the tolerance parameter $\tau$. When incorporating the tolerance $\tau$ again, IPRO stops when the approximation error is at most $\tau$, as guaranteed by \cref{th:approximation-guarantee}, therefore resulting in a $\tau$-Pareto front.

Let us first show that the sequence of errors $(\varepsilon_t)_{t \in \mathbb{N}}$ converges to zero. From \cref{def:approximate-pareto-oracle}, for an approximate Pareto oracle with tolerance $\tau$ and inherent tolerance $\bar\tau$ we know that $\tau > \bar\tau$. Let $\lset_t$ be the lower bounds in timestep $t$ and select $\vl$ from it as the referent. When the oracle evaluation is unsuccessful, $\vl$ is removed from the set of lower bounds as well as any upper bound that is not on the reachable boundary anymore. When the oracle evaluation is successful, a finite number of new lower and upper bounds are added. Importantly, each lower bound is a $\tau - \bar\tau$ improvement in one dimension and \cref{lemma:upper} guarantees that each new upper bound is dominated by the old bound. Repeating this process across multiple iterations, we can consider the sequence of lower bounds spawned from the root lower bound as a \emph{tree} where eventually each branch will be closed when at the leaf $\vl$ we have that $\vl + \tau$ is in the infeasible set or out of the bounding box. 

Recall that the set of lower bounds at timestep $t$ is defined as $\lset_t = \left\{\vl \in \rboundary \dset_t \mid \nexists \vv \in \rboundary \dset_t, \vl \pd \vv \right\}$ and the set of upper bounds as $\uset_t = \left\{\vu \in \rboundary \iset_t \mid \nexists \vv \in \rboundary \iset_t, \vv \pd \vu \right\}$. Observe that this implies that when the set of lower bounds is empty, this implies that $\rboundary \dset = \rboundary \iset = \emptyset$ and therefore the set of upper bounds must be empty as well. As such, IPRO with an approximate Pareto oracle removes all upper bounds and the sequence of errors $(\varepsilon_t)_{t \in \mathbb{N}}$ converges to zero. Furthermore, this must occur at some timestep $t < \infty$, since there is a minimal improvement at each timestep of $\tau - \bar\tau$. Using \cref{th:approximation-guarantee}, we have that the true approximation error is upper bounded by $\varepsilon_t$. As IPRO terminates when this upper bound is at most equal to the tolerance $\tau$, it is guaranteed to converge to a $\tau$-Pareto front.
\end{proof}

Finally, we provide an analogous result for weak Pareto oracles and show that they almost surely converge to the exact Pareto front in the limit. The probabilistic nature of this result is necessary to handle stochastic referent selection as well as oracle evaluations.
\begin{theorem}
\label{th:weak-po-convergence}
Given a weak Pareto oracle $\oracle$ with tolerance $\tau = 0$, IPRO converges almost surely to the Pareto front when $t \to \infty$.
\end{theorem}
\begin{proof}
We first show that the sequence of errors $(\varepsilon_t)_{t \in \mathbb{N}}$ has its infimum at zero with probability $1$. By contradiction, assume that this is not the case and it has its infimum instead at some $\beta > 0$. As a consequence of \cref{th:approximate-po-convergence}, for any $0 < \varepsilon < \beta$ there must be some finite set $\vset^{\varepsilon}$ of Pareto optimal points such that $\max_{\vu \in \uset} \min_{\vv \in \vset^\varepsilon} \|\vu - \vv\|_\infty \leq \varepsilon$. From \cref{assumption:weak-oracle-strength}, we know that there is some lower bound $\vl$ for every $\vast \in \vset^{\varepsilon}$ such that $\oracle(\vl) = \vast$ and that as $t \to \infty$ the probability that $\vl$ is added to $\lset_t$ is $1$. Furthermore, from \cref{assumption:weak-oracle-referent-selection} we have that the probability that $\vl$ gets selected as the lower bound is $1$. As such, $\forall \vast \in \vset^\varepsilon: \lim_{t \to \infty} \Prob\left(\vast \in \vset_t\right) = 1$ which implies $\lim_{t \to \infty} \Prob\left(\vset^\varepsilon \subseteq \vset_t\right) = 1$ and therefore there is no lower bound $\beta > 0$ of $(\varepsilon_t)_{t \in \mathbb{N}}$ with probability $1$. Since we have $\Prob\left(\inf_t \{\varepsilon_t\} = 0\}\right) = 1$, using \cref{co:decreasing-error} and the monotone convergence theorem we get, $\Prob\left(\lim_{t \to \infty} \{\varepsilon_t\} = 0\right) = 1$, thereby showing that IPRO converges almost surely to the Pareto front.
\end{proof}

\subsection{Proof of Theorem 4.3}
\label{sec:proof-th43}
We conclude the theoretical results for IPRO by analysing the runtime complexity for the number of iterations in IPRO with $\tau > 0$. We first restate \cref{th:ipro-complexity} below and subsequently provide a proof.

\begin{theorem43}
Given a Pareto oracle $\oracle$ and tolerance $\tau > 0$, let $\forall j \in [d], k_j = \ceil{\nicefrac{(v^\text{i}_j - v^\text{n}_j)}{\tau}}$. IPRO constructs a $\tau$-Pareto front in at most
\begin{equation}
    \prod_{j=1}^d k_j - \prod_{j=1}^d (k_j - 1)
\end{equation}
iterations which is a polynomial in $\tau$ but exponential in the number of objectives $d$.
\end{theorem43}
\begin{proof}
We consider an approximate Pareto oracle with tolerance $\tau > 0$ and inherent tolerance $\bar{\tau} = 0$. Notice that this is without loss of generality, since for any $\bar{\tau} > 0$ it is assumed that $\tau > \bar{\tau}$ and therefore we may use $\tau - \bar{\tau} > 0$ instead. Given a bounding box $\bbox$ with nadir $\nadir$ and ideal $\ideal$, we discretise the space into hypercubes of side length $\tau$. This results in $\prod_{j=1}^d k_j$ cells, where $k_j = \ceil{\nicefrac{(\ideal_j - \nadir_j)}{\tau}}$. While IPRO determines the discretisation adaptively, this predefined discretisation serves as an upper bound on the number of cells.

The proof is composed of three parts. First, we show that it is always possible to construct a worst-case sequence for IPRO consisting of only successful Pareto oracle evaluations. We then use this property to build a worst-case layout of the Pareto front by placing Pareto optimal solutions along the Pareto optimal facets of the bounding box. Lastly, we derive the formula to compute the number of cells on these Pareto optimal facets.

To establish an upper bound on the number of iterations, consider the worst-case layout of the Pareto front and the corresponding sequence of lower bounds. Let $\left\{\left(\vl_t, \vv_t\right)\right \}_{n}$ be a finite sequence of lower bounds generated by IPRO, leading to a $\tau$-Pareto front, where at timestep $\bar{t}$, no Pareto optimal solution was found. Define a new sequence where the lower bound at $\bar{t}$ is proposed only at the last referent: $\left\{\left(\vl'_t, \vv'_t\right)\right \}_{n}$, with $\forall 0 \leq t < \bar{t}$, $\vl_t = \vl'_t$ and $\vv = \vv'$, and for all $\bar{t} \leq t \leq n$, $\vl_{t+1} = \vl'_{t}$ and $\vv_{t+1} = \vv'_{t}$. Finally, $\vl_{\bar{t}} = \vl'_{n}$ and $\vv_{\bar{t}} = \vv'_{n}$. This new sequence remains valid, as an unsuccessful evaluation at referent $\vl_{\bar{t}}$ only removes it from $\lset$. Now, if $\vl'_{\bar{n}}$ instead finds a Pareto optimal solution, new lower bounds might be introduced, requiring additional iterations. Therefore, it is always possible to construct a worst-case sequence of lower bounds that only consists of successful evaluations.

The problem then reduces to finding the size of the largest set of cells that contain a Pareto optimal solution, which is equivalent to the largest set of cells where no cell strictly dominates another. Let $\sS$ be such a maximal set. We define an operation that produces a new set $\sS'$ with equal cardinality and no strict dominance. Specifically, for any cell $\vc \in \sS$, if $\vc + 1 \in \bbox$, we add $\vc + 1$ to $\sS'$; otherwise, we add $\vc$. It remains to show that $\forall \vc'_1 \in \sS', \nexists \vc'_2 \in \sS'$ such that $\vc'_2 > \vc'_1$ and $|\sS| = |\sS'|$. Suppose that there is a $\vc'_1, \vc'_2 \in \sS'$ such that $\vc'_2 > \vc'_1$. There are four distinct scenarios:
\begin{enumerate}
    \item $\vc_1 = \vc'_1$ and $\vc_2 = \vc'_2$
    \item $\vc_1 + 1 = \vc'_1$ and $\vc_2 = \vc'_2$
    \item $\vc_1 = \vc'_1$ and $\vc_2 + 1 = \vc'_2$
    \item $\vc_1 + 1 = \vc'_1$ and $\vc_2 + 1 = \vc'_2$
\end{enumerate} 

By the definition of $\sS$, (1), (2), and (4) cannot occur since $\vc_2 > \vc_1$ would otherwise hold. Moreover, (3) is also impossible, as $\vc'_1$ lies on the boundary of $\bbox$ and the only points that could strictly dominate $\vc'_1$ are outside $\bbox$.

Finally, we demonstrate that $|\sS| = |\sS'|$. Since $\sS$ was constructed as a set of maximal size that contains no cell strictly dominating another, and since $\sS'$ does not contain such dominance, we have $|\sS| \geq |\sS'|$. Suppose $|\sS| > |\sS'|$. Then, at least two cells, $\vc_1, \vc_2 \in \sS$, must map to the same $\vc' \in \sS'$. Since $\vc_1 \neq \vc_2$, we have $\vc_1 + 1 \neq \vc_2 + 1$, implying $\vc_1 = \vc_2 + 1$ or $\vc_2 = \vc_1 + 1$. However, this leads to a contradiction as $\vc_1 > \vc_2$ or $\vc_2 > \vc_1$.

Repeatedly applying this operation yields a fixed point $\sS^\ast$, where all $\vc^\ast \in \sS^\ast$ lie on the Pareto optimal facets of $\bbox$. The number of iterations required is then equal to the number of cells on these facets. We can compute this as the total number of cells in $\bbox$ minus those that are dominated, leading to the final equation: \begin{equation} 
\prod_{j=1}^d k_j - \prod_{j=1}^d (k_j - 1), 
\end{equation} 
which is polynomial in $\tau$ and exponential in $d$. 
\end{proof}

\section{Theoretical results for Pareto oracles}
\label{ap:po-proofs}
We present formal proofs for the theoretical results in \cref{sec:pareto-oracle}. These results develop the concept of a Pareto oracle and relate it to achievement scalarising functions. While we utilise Pareto oracles as a subroutine in IPRO to provably obtain a Pareto front, they may also be of independent interest in other settings. We further contribute alternative implementations of a Pareto oracle, thus demonstrating their applicability beyond the ASFs considered in this work.

\subsection{Pareto oracles from achievement scalarising functions}
\label{ap:po-asf}
In \cref{sec:pareto-oracle} we defined Pareto oracles and subsequently related them to achievement scalarising functions. Here, we provide formal proof of the established connections. To establish the notation, let $\sX$ be the set of feasible solutions and define a mapping $f: \sX \to \mathbb{R}^d$ which maps a solution to its $d$-dimensional return. Let us further define the Euclidean distance function between a point $\vv \in \mathbb{R}^d$ and a set $\sY \subseteq \mathbb{R}^d$ as $\dist (\vv, \sY) = \inf_{\vy \in \sY}\|\vv - \vy\|$. Finally, let $\rapprox = \{\vv \in \mathbb{R}^d \mid \dist(\vv, \rpos) \leq \delta \|\vv\|\}$, where $\delta$ is a fixed scalar in $[0, 1)$. Using this notation, we define both order representing and order approximating ASFs following the formalisation by \citet{miettinen1998nonlinear}. 

\begin{definition}
\label{def:order-representing}
We say an ASF $\sr: \mathbb{R}^d \to \mathbb{R}$ is order representing when $\forall \vr \in \mathbb{R}^d, \forall x, y \in \sX$ with $f(x) = \vx$ and $f(y) = \vy$, $\sr$ is strictly increasing such that $\vx > \vy \implies \sr(\vx) > \sr(\vy)$. In addition, $\sr(\vr) = 0$ and
\begin{equation}
\label{eq:order-representing-non-negative}
    \{\vv \in \mathbb{R}^d \mid \sr(\vv) \geq 0 \} = \vr + \rpos.
\end{equation}
\end{definition}

\begin{definition}
\label{def:order-approximating}
We say an ASF $\sr: \mathbb{R}^d \to \mathbb{R}$ is order approximating when $\forall \vr \in \mathbb{R}^d, \forall x, y \in \sX$ with $f(x) = \vx$ and $f(y) = \vy$, $\sr$ is strongly increasing such that $\vx \pd \vy \implies \sr(\vx) > \sr(\vy)$. In addition, $\sr(\vr) = 0$ and with $\delta > \bar\delta \geq 0$
\begin{equation}
\label{eq:order-approximating-non-negative}
    \vr + \mathbb{R}^d_{\bar{\delta}} \subset \{\vv \in \mathbb{R}^d \mid \sr(\vv) \geq 0 \} \subset \vr + \rapprox.
\end{equation}
\end{definition}

These definitions can be applied to the reinforcement learning setting where the set of feasible solutions is a policy class $\policies$ and the quality of a policy $\policy \in \policies$ is determined by its expected return $\vpi$. Using these definitions, we provide a formal proof for \cref{th:weak-po} which we first restate below.
\begin{theorem53}
Let $\sr$ be an order representing ASF. Then $\oracle(\vr) = \argmax_{\policy \in \policies} \sr(\vpi)$ with tolerance $\tau = 0$ is a valid weak Pareto oracle.
\end{theorem53}
\begin{proof}
Let $\sr$ be an order representing achievement scalarising function and define a Pareto oracle $\mathcal{O}: \mathbb{R}^d \to \policies$ such that, $\mathcal{O}(\vr) = \argmax_{\policy \in \policies} \sr(\vpi) = \policy^\ast$. Denote the expected return of $\policy^\ast$ as $\vv^\ast$. We first consider the case when $\vast \npde \vr$. By \cref{eq:order-representing-non-negative} this implies that $\sr(\vast) < 0$. This guarantees that no feasible weakly Pareto optimal policy $\policy'$ exists with expected return $\vv'$ such that $\vv' \pde \vr$, as otherwise $\sr(\vv') \geq 0 > \sr(\vast)$ and thus $\policy^\ast$ would not have been returned as the maximum. 

We now consider the case when $\vast \pde \vr$. Then $\policy^\ast$ is guaranteed to be weakly Pareto optimal. By contradiction, if $\policy^\ast$ is not weakly Pareto optimal, another policy $\policy'$ exists such that $\vv' > \vast$. However, this would imply that $\sr(\vv') > \sr(\vast)$ and thus $\policy^\ast$ would not have been returned as the maximum.
\end{proof}

We provide a similar result using order approximating ASFs instead. While such ASFs enable the Pareto oracle to return Pareto optimal solutions rather than only weakly optimal solutions, the quality of the oracle with respect to the target region becomes dependent on the approximation parameter $\delta$ of the ASF. The core idea in the proof of \cref{th:approx-po} is that we can define a lower bound on the shift necessary to ensure only feasible solutions in the target region have a non-negative value. When feasible solutions exist in the shifted target region, we can then conclude by the strongly increasing property of the ASF that the maximum is Pareto optimal. 

\begin{theorem54}
Let $\sr$ be an order approximating ASF and let $\vl \in \mathbb{R}^d$ be a lower bound such that only referents $\vr$ are selected when $\vr \pde \vl$. Then $\sr$ has an inherent oracle tolerance $\bar\tau > 0$ and for any user-provided tolerance $\tau > \bar\tau$, $\oracle(\vr) = \argmax_{\policy \in \policies} s_{\vr + \tau}(\vpi)$ is a valid approximate Pareto oracle.
\end{theorem54}
\begin{proof}
Let $\vl$ be the lower bound for all referents $\vr$. We define $\bar\tau$ to be the minimal shift such that all feasible solutions with non-negative values for an order approximating ASF $s_{\vl + \bar\tau}$ with the shifted referent $\vl + \bar\tau$ are inside the box $B(\vl, \ideal)$ defined by the lower bound and ideal. The lower bound on $\bar\tau$ is clearly zero which implies that no shift is necessary. We now define an upper bound for this shift which ensures that no feasible solution has a non-negative value except potentially $\vl$ itself.

Recall the definition of $\rapprox = \{\vv \in \mathbb{R}^d \mid \dist(\vv, \rpos) \leq \delta \|\vv\|\}$, where $\delta$ is a fixed scalar in $[0, 1)$. We refer to $\vl + \rapprox$ as the extended target region. Suppose there exists a point in this extended target region $\vv \in \vl + \rapprox$ such that $\vl \pd \vv$. This implies we can write $\vv = \vl + \vx$, where $\vx$ is a non-positive vector. However, this then further implies that, $\text{dist}(\vx, \mathbb{R}^d_{\geq 0}) = \inf_{\vect{s} \in \mathbb{R}^d_{\geq 0}}\|\vx - \vect{s}\| = \|\vx\|$ as $0$ is the closest point in $\mathbb{R}^d_{\geq 0}$ for a non-positive vector. However, for $\delta \in [0, 1)$ it cannot be true that $\|\vx\| \leq \delta \|\vx\|$. Therefore, there exists no point in $\vl + \rapprox$ that is dominated by $\vl$. As such, for all points $\vv$ in the extended target region that are not equal to $\vl$, there must be a dimension $j \in \{1, \dotsc, d\}$ such that $v_j > l_j$. Consider now the shift imposed by the $\normmax$ distance between the lower point $\vl$ and ideal $\ideal$. This ensures that all points in the extended target region except $\vl$ are strictly above the ideal in at least one dimension, further implying that they are infeasible by the definition of the ideal. As such, $\|\ideal - \vl\|_\infty$ is an upper bound for $\bar\tau$.

Let us now formally define $\bar\tau$ for an order approximating ASF with approximation constant $\delta$,
\begin{equation}
\label{eq:lower-bound-epsilon}
    \bar\tau = \inf \left\{0 < \tau \leq \|\ideal - \vl\|_\infty \mid \left(\vl + \tau + \rapprox \right) \cap \{\vv \in \mathbb{R}^d \mid \ideal \pde \vv\} \subseteq B(\vl, \ideal)\right\}.
\end{equation}

In \cref{fig:proof-th32} we illustrate that this shift ensures all feasible solutions with non-negative values are inside the box. Observe, however, that by the nature of this shift, it can also ensure that some feasible solutions in the bounding box are excluded from the non-negative set. 

\begin{figure}[htb]
    \centering
        \begin{subfigure}[b]{0.3\textwidth}
        \centering
        \includegraphics[width=\textwidth]{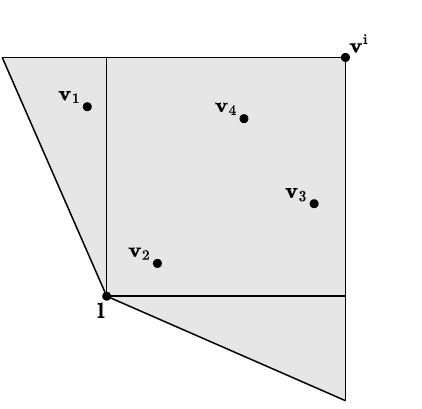}
        \subcaption{}
        \label{fig:proof-shift1}
    \end{subfigure}
    \qquad
    \begin{subfigure}[b]{0.3\textwidth}
        \centering
        \includegraphics[width=\textwidth]{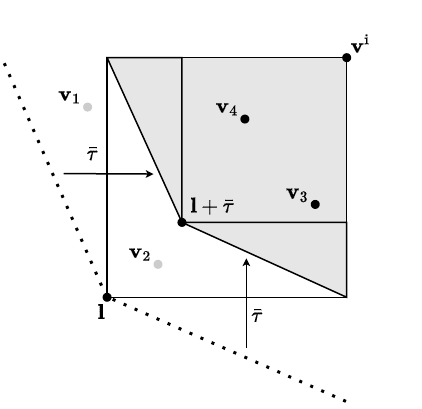}
        \subcaption{}
        \label{fig:proof-shift2}
    \end{subfigure}
    \caption{(a) A possible non-negative set (shaded) for an order approximating ASF with referent $\vl$. (b) Shifting $\vl$ by $\bar\tau$ ensures that all feasible solutions with non-negative values are in the box $B(\vl, \ideal)$.}
    \label{fig:proof-th32}
\end{figure}

Let us now show that the Pareto oracle $\oracle(\vr) = \argmax_{\policy \in \policies} s_{\vr + \tau}(\vpi)$ with $\tau > \bar\tau$ functions as required for the referent $\vl$. Assume there exists a Pareto optimal optimal $\policy'$ with expected return $\vv'$ such that $\vv' \pde \vl + \tau$. Then $s_{\vl + \tau}(\vv') \geq 0$ and therefore the maximisation will return a non-negative solution $\policy^\ast$ with expected returns $\vast$. By the definition of $\bar\tau$ we know that all feasible solutions $\policy$ with non-negative value $s_{\vl + \tau}(\vpi)$ satisfy the condition $\vpi \pde \vl + (\tau - \bar\tau)$ and therefore $\vast \pde \vl + (\tau - \bar\tau)$. Moreover, as the ASF is guaranteed to be strongly increasing, there exists no policy $\policy$ such that $\vv^{\policy} \pd \vast$ and therefore $\policy^\ast$ is Pareto optimal.


Given the lower bound $\vl$, for all referents $\vr$ such that $\vr \pde \vl$ and with $\tau > \bar\tau$, the Pareto oracle remains valid. To see this, observe that $\vr = \vl + \vx$ where $\vx$ is now a non-negative vector. Then,
\begin{equation}
\begin{split}
& \left(\vl + \tau + \rapprox \right) \cap \{\vv \in \mathbb{R}^d \mid \ideal \pde \vv\} \subseteq B(\vl, \ideal) \\
\implies & \left(\vl + \tau + \rapprox \right) \cap \{\vv \in \mathbb{R}^d \mid \ideal - \vx \pde \vv\} \subseteq B(\vl, \ideal - \vx).
\end{split}
\end{equation}
This implication can be shown by contradiction. Assume that,
\begin{equation}
    \exists \vv \in \left(\vl + \tau + \rapprox \right) \cap \{\vv \in \mathbb{R}^d \mid \ideal - \vx \pde \vv\} \text{ and } \vv \notin B(\vl, \ideal - \vx).
\end{equation}
However, by definition of $\vv$, $\ideal - \vx \pde \vv$ and
\begin{align*}
    &\vv \in \left(\vl + \tau + \rapprox \right) \cap \{\vv \in \mathbb{R}^d \mid \ideal - \vx \pde \vv\} \\
    \implies & \vv \in \left(\vl + \tau + \rapprox \right) \cap \{\vv \in \mathbb{R}^d \mid \ideal\pde \vv\}\\
    \implies & \vv \in B(\vl, \ideal) \\
    \implies & \vv \pde \vl.
\end{align*}
As $\vv \pde \vl$ and $\ideal - \vx \pde \vv$ this implies $\vv \in B(\vl, \ideal - \vx)$, which is a contradiction. Therefore $\left(\vl + \tau + \rapprox \right) \cap \{\vv \in \mathbb{R}^d \mid \ideal - \vx \pde \vv\} \subseteq B(\vl, \ideal - \vx)$. By a rigid transformation and recalling that $\vr = \vl + \vx$, we obtain,
\begin{equation}
    \left(\vr + \tau + \rapprox \right) \cap \{\vv \in \mathbb{R}^d \mid \ideal \pde \vv\} \subseteq B(\vr, \ideal).
\end{equation}
We can subsequently apply the same reasoning to establish the validity of the Pareto oracle for the lower bound $\vl$ to all dominating referents $\vr$.
\end{proof}

\subsection{Alternative Pareto oracles}
\label{ap:alternative-po}
To conclude the theoretical results for Pareto oracles, we demonstrate that both convex MDPs and constrained MDPs may be leveraged to implement them.

\smallparagraph{Convex MDPs} A convex Markov decision process is a generalisation of an MDP, where an agent seeks to minimise a convex function (or equivalently maximise a concave function) over a convex set of admissible occupancy measures. Let $\mathcal{K}_\gamma$ be the set of discounted state occupancy measures for some discount factor $\gamma$. The expected return $\vpi$ of some policy $\pi$ can be written as a linear function of the occupancy measure of the policy $d_\pi$ and the reward function of the MDP, $\vpi = \sum_{\s, \action} \vecrewards(\s, \action) d_\pi(\s, \action)$. \Cref{prop:convex-mdp-oracle} then follows immediately.

\begin{proposition}
\label{prop:convex-mdp-oracle}
Let $\momdp = \momdptuple$ be a MOMDP with $d$ objectives. For a given oracle tolerance $\tau \geq 0$ and referent $\vr$, we define a convex MDP $\mdp_{\text{conv}}$ with the same states, actions, transition function, discount factor and initial state distribution as $\momdp$. For $\sr$ defined in \cref{eq:aasf} and $\Pi$ the set of stochastic policies, $\oracle(\vr) = \argmax_{d_\policy \in \mathcal{K}_\gamma} s_{\vr + \tau}(\vpi)$ is a valid weak or approximate Pareto oracle.
\end{proposition}
\begin{proof}
Since \cref{eq:aasf} is concave for any referent $\vr$ and the composition of a linear function and concave function preserves concavity, the problem is concave. Furthermore, $\mathcal{K}_\gamma$ is by definition a convex polytope for the set of stochastic policies. As such, $\mdp_{\text{conv}}$ is a convex MDP and since $\sr$ can be constructed as both an order representing and order approximating achievement scalarising function, \cref{th:weak-po} and \cref{th:approx-po} can be applied.
\end{proof}

This reformulation enables the use of techniques with strong theoretical guarantees. For instance, \citet{zhang2020variational} propose a policy gradient method that converges to the global optimum, and \citet{zahavy2021reward} introduce a meta-algorithm using standard RL algorithms that converges to the optimal solution with any tolerance, assuming reasonably low-regret algorithms. Additionally, it has been demonstrated that for any convex MDP, a mean-field game can be constructed, for which any Nash equilibrium in the game corresponds to an optimum in the convex MDP \citep{geist2022concave}.

\smallparagraph{Constrained MDPs}
A constrained Markov decision process $\mdp_{\text{const}}$ is an MDP, augmented with a set of $m$ auxiliary cost functions $C_j: \states \times \actions \times \states \to \mathbb{R}$ and related limit $c_j$. Let $J_{C_j}(\policy)$ denote the expected discounted return of policy $\policy$ for the auxiliary cost function $C_j$. The feasible policies from a given class of policies $\policies$ is then $\policies_C = \{\policy \in \policies \mid \forall i, J_{C_j}(\policy) \geq c_j\}$. Finally, the reinforcement learning problem in a CMDP is as follows,
\begin{equation}
    \policy^\ast = \argmax_{\policy \in \policies_C}v^\policy.
\end{equation}

We demonstrate that an approximate Pareto oracle can be implemented by solving an auxiliary constrained MDP, where the constraints ensure that the target region is respected and the scalar reward function is designed such that only Pareto optimal policies are returned as the optimal solution. Importantly, since constrained MDPs have no inherent tolerance, the user is free to select any tolerance $\tau > 0$.

\begin{proposition}
\label{prop:cmdp-oracle}
Let $\momdp = \momdptuple$ be a MOMDP with $d$ objectives. For a given oracle tolerance $\tau > 0$ and referent $\vr$, we define a constrained MDP $\mdp_{\text{const}}$ with the same states, actions, transition function, discount factor and initial state distribution as $\momdp$. $\mdp_{\text{const}}$ has $d$ cost functions corresponding to the original $d$ reward function with limits $\vr + \tau$ and the scalar reward function is the sum of the original reward vector. Then $\oracle(\vr) = \argmax_{\policy \in \policies_C}v^\policy$ is a valid approximate Pareto oracle.
\end{proposition}
\begin{proof}
Assume the construction outlined in the theorem and that there exists a Pareto optimal policy $\policy$ such that $\vpi \pde \vr + \tau$. Then $\policies_C$ is non-empty and the Pareto oracle $\mathcal{O}^\tau(\vr) = \argmax_{\policy \in \policies_C}v^\policy$ returns a Pareto optimal policy $\policy^\ast$ with expected return $\vast$ such that $\vast \pde \vr + \tau$. If $\policy^\ast$ is not Pareto optimal, there exists a policy $\policy'$ with expected return $\vv'$ such that $\vv' \pd \vast$. This then implies that,
\begin{equation}
    \sum_{j \in \{1, \dots, d\}} v'_j > \sum_{j \in \{1, \dots, d\}} v^{\ast}_j
\end{equation}
which leads to a contradiction.
\end{proof}

\section{Experiment details}
\label{ap:experiment-details}
In this section, we provide details concerning the experimental evaluation presented in \cref{sec:experiments}. All experiments were run on a single CPU core with access to at most 4GB of RAM and were allowed a maximum of three days wallclock time.

\subsection{Environments}
We initialise each experiment with predefined minimal and maximal points to establish the bounding box of the environment. It is important to emphasise that these points can be obtained using conventional reinforcement learning algorithms without requiring any modifications, justifying their omission from our evaluation process.

\smallparagraph{Deep Sea Treasure (DST)} We initialise IPRO with $(124, -50)$ and $(1, -1)$ as the maximal points and give $(0, -50)$ as the only minimal point. We set the discount factor to $1$, signifying no discounting, and maintain a fixed time horizon of $50$ timesteps for each episode. We note that we one-hot encode the observations due to the discrete nature of the state space. Finally, a tolerance $\tau$ of $0$ was set to allow IPRO to find the complete Pareto front in this environment.

\smallparagraph{Minecart} In the Minecart environment, we set $\gamma = 0.98$ to align with related work. For minimal points, IPRO is initialised with the nadir $(-1, -1, -200)$ for each dimension. For maximal points, we consider the nadir and set each dimension to its theoretical maximum: $(1.5, -1, -200), (-1, 1.5, -200), (-1, -1, 0)$. Our reference point is also the nadir and the time horizon is 1000. A tolerance of $1 \times 10^{-15}$ was used.

\smallparagraph{MO-Reacher} In the Reacher environment, we use $(-50, -50, -50, -50)$ in each dimension as the minimal points, and similarly, set this vector to 40 for each dimension for the maximal points. The discount factor $\gamma$ is set to $0.99$. The reference point is again set to the nadir, a time horizon of 50 was used and tolerance was set to $1 \times 10^{-15}$.


\subsection{Hyperparameters}
\label{sec:hyperparameters}
In \cref{tab:hp-desc} we provide a description of all hyperparameters used in our Pareto oracles and the algorithms for which they apply. Finally, in \cref{tab:hp-dqn,tab:hp-a2c,tab:hp-ppo} we give the hyperparameter values used in our reported experiments.

\begin{table}[H]
\caption{A description of the relevant hyperparameters.}
\label{tab:hp-desc}
\centering
\begin{tabular}{lll}
\multicolumn{1}{c}{\bf Parameter} &\multicolumn{1}{c}{\bf{Algorithm}}  &\multicolumn{1}{c}{\bf{Description}} \\ \hline \\
scale     & DQN, A2C, PPO    & Scale the output of \cref{eq:aasf} \\
$\rho$ & DQN, A2C, PPO & Augmentation parameter from \cref{eq:aasf} \\
pretrain\_iters & DQN, A2C, PPO & The number of pretraining iterations \\
num\_referents & DQN, A2C, PPO & The number of additional referents to sample for pretraining \\
pretraining\_steps & DQN, A2C, PPO & Number of global steps while pretraining \\
online\_steps & DQN, A2C, PPO & Number of global steps while learning online \\
critic\_hidden & DQN, A2C, PPO & Number of hidden neurons per layer for the critic\\
lr\_critic & DQN, A2C, PPO & The learning rate for the critic\\
actor\_hidden & A2C, PPO & Number of hidden neurons per layer for the actor\\
lr\_actor & A2C, PPO & Learning rate for the actor \\
n\_steps & A2C, PPO & Number of environment interactions before each update \\
gae\_lambda & A2C, PPO & $\lambda$ parameter for generalised advantage estimation \\
normalise\_advantage & A2C, PPO & Normalise the advantage\\
e\_coef & A2C, PPO & Entropy loss coefficient to compute the overall loss \\ 
v\_coef & A2C, PPO & Value loss coefficient to compute the overall loss \\
max\_grad\_norm & A2C, PPO & Maximum gradient norm\\
clip\_coef & PPO& Clip coefficient used in the PPO surrogate objective \\
num\_envs & PPO & Number of parallel environments to run in\\
clip\_range\_vf & PPO & Clipping range for the value function\\
update\_epochs & PPO & Number of update epochs to execute\\
num\_minibatches & PPO & Number of minibatches to divide a batch in\\
batch\_size & DQN & Batch size for each update \\
buffer\_size & DQN & Size of the replay buffer \\
soft\_update & DQN & Multiplication factor for the soft update\\
pre\_epsilon\_start & DQN & Pretraining starting exploration probability \\
pre\_epsilon\_end & DQN & Pretraining final exploration probability \\
pre\_exploration\_frac & DQN & Pretraining exploration fraction of total timesteps \\
pre\_learning\_start & DQN & Pretraining start of learning \\
online\_epsilon\_start & DQN & Online starting exploration probability \\
online\_epsilon\_end & DQN & Online final exploration probability \\
online\_exploration\_frac & DQN & Online exploration fraction of total timesteps \\
online\_learning\_start & DQN & Online start of learning \\
\end{tabular}
\end{table}

\begin{table}[H]
\caption{The hyperparameters used in the DQN oracles.}
\label{tab:hp-dqn}
\centering
\begin{tabular}{llll}
\multicolumn{1}{c}{\bf Parameter} &\multicolumn{1}{c}{\bf{DST}}  &\multicolumn{1}{c}{\bf{Minecart}} &\multicolumn{1}{c}{\bf{MO-Reacher}} \\ \hline \\
scale                     & 100 & 100 & 10 \\
$\rho$                    & 0.1 & 0.01 & 0.01 \\
pretrain\_iters           & / & 50 & 50 \\
num\_referents            & / & 32 & 16 \\
online\_steps             & 2.5e+04 & 2.0e+04 & 7.5e+03 \\
pretraining\_steps        & / & 2.0e+04 & 7.5e+03 \\
critic\_hidden            & (256, 256) & (256, 256, 256, 256) & (256, 256, 256, 256) \\
lr\_critic                & 0.0003 & 0.0001 & 0.0007 \\
batch\_size               & 512 & 32 & 16 \\
buffer\_size              & 1.0e+04 & 1.0e+05 & 1.0e+05 \\
soft\_update              & 0.25 & 0.1 & 0.1 \\
pre\_learning\_start      & / & 1.0e+03 & 1.0e+03 \\
pre\_epsilon\_start       & / & 0.75 & 0.5 \\
pre\_epsilon\_end         & / & 0.2 & 0.1 \\
pre\_exploration\_frac    & / & 0.75 & 0.75 \\
online\_learning\_start   & 100 & 100 & 100 \\
online\_epsilon\_start    & 1 & 0.5 & 0.5 \\
online\_epsilon\_end      & 0.05 & 0.1 & 0.05 \\
online\_exploration\_frac & 0.75 & 0.25 & 0.5
\end{tabular}
\end{table}

\begin{table}[H]
\caption{The hyperparameters used in the A2C oracles.}
\label{tab:hp-a2c}
\centering
\begin{tabular}{llll}
\multicolumn{1}{c}{\bf Parameter} &\multicolumn{1}{c}{\bf{DST}}  &\multicolumn{1}{c}{\bf{Minecart}} &\multicolumn{1}{c}{\bf{MO-Reacher}} \\ \hline \\
scale                & 100 & 100 & 100 \\
$\rho$               & 0.01 & 0.01 & 0.01 \\
pretrain\_iters      & 75 & 75 & 75 \\
num\_referents       & 16 & 16 & 16 \\
online\_steps        & 5.0e+03 & 2.5e+04 & 5.0e+03 \\
pretraining\_steps   & 2.5e+03 & 7.5e+04 & 2.5e+04 \\
critic\_hidden       & (128,) & (128, 128, 128) & (64, 64) \\
lr\_critic           & 0.001 & 0.0001 & 0.0007 \\
actor\_hidden        & (128,) & (128, 128, 128) & (64, 64) \\
lr\_actor            & 0.0001 & 0.0001 & 0.001 \\
n\_steps             & 16 & 32 & 16 \\
gae\_lambda          & 0.95 & 0.95 & 0.95 \\
normalise\_advantage & False & False & False \\
e\_coef              & 0.01 & 0.1 & 0.1 \\
v\_coef              & 0.5 & 0.5 & 0.1 \\
max\_grad\_norm      & 0.5 & 50 & 1
\end{tabular}
\end{table}

\begin{table}[H]
\caption{The hyperparameters used in the PPO oracles.}
\label{tab:hp-ppo}
\centering
\begin{tabular}{llll}
\multicolumn{1}{c}{\bf Parameter} &\multicolumn{1}{c}{\bf{DST}}  &\multicolumn{1}{c}{\bf{Minecart}} &\multicolumn{1}{c}{\bf{MO-Reacher}} \\ \hline \\
scale                & 100 & 100 & 100 \\
$\rho$               & 0.1 & 0.01 & 0.01 \\
pretrain\_iters      & / & 100 & 100 \\
num\_referents       & / & 32 & 8 \\
online\_steps        & 3.0e+04 & 2.5e+04 & 7.5e+03 \\
pretraining\_steps   & / & 2.0e+04 & 1.5e+04 \\
critic\_hidden       & (128, 128) & (256, 256) & (128, 128, 128) \\
lr\_critic           & 0.0006 & 0.0001 & 0.001 \\
actor\_hidden        & (128, 128) & (256, 256) & (128, 128, 128) \\
lr\_actor            & 0.0002 & 0.0001 & 0.0003 \\
n\_steps             & 16 & 32 & 64 \\
gae\_lambda          & 0.95 & 0.95 & 0.95 \\
normalise\_advantage & False & False & False \\
e\_coef              & 0.05 & 0.1 & 0.1 \\
v\_coef              & 0.5 & 0.1 & 0.1 \\
max\_grad\_norm      & 5 & 50 & 5 \\
clip\_coef           & 0.2 & 0.1 & 0.2 \\
num\_envs            & 8 & 16 & 8 \\
anneal\_lr           & False & False & False \\
clip\_range\_vf      & 0.2 & 0.5 & 0.2 \\
update\_epochs       & 2 & 16 & 4 \\
num\_minibatches     & 4 & 4 & 4
\end{tabular}
\end{table}

\end{document}